\documentclass{article}

\newif\ifarxiv       \arxivtrue
\newif\ifcameraready \camerareadyfalse

\usepackage{iclr2027_conference,times}

\usepackage{microtype}
\usepackage{graphicx}
\usepackage{subcaption}
\usepackage{array}    %
\usepackage{booktabs} %
\usepackage{dsfont}

\usepackage{hyperref}

\newcommand{\init}{\ensuremath{\mathtt{init}}}
\newcommand{\query}{\ensuremath{\mathtt{query}}}
\newcommand{\step}{\ensuremath{\mathtt{step}}}
\newcommand{\reachable}{\ensuremath{\mathtt{reachable}}}
\newcommand{\reachableWith}{\ensuremath{\mathtt{reachableWith}}}
\newcommand{\unif}{\ensuremath{\mathrm{unif}}}
\newcommand{\Good}{\ensuremath{\mathrm{Good}}}
\newcommand{\Worsen}{\ensuremath{\mathrm{Worsen}}}
\newcommand{\BigWorsen}{\ensuremath{\mathrm{BigWorsen}}}
\newcommand{\Jump}{\ensuremath{\mathrm{Jump}}}

\usepackage{longtable}

\ifarxiv\iclrfinalcopy\fi
\ifcameraready\iclrfinalcopy\fi

\usepackage{amsmath}
\usepackage{amssymb}
\usepackage{mathtools}
\usepackage{amsthm}

\DeclareMathOperator{\Unif}{Unif}

\usepackage[capitalize,noabbrev]{cleveref}
\makeatletter
\g@addto@macro\appendix{%
  \crefalias{section}{appendix}%
  \crefalias{subsection}{subappendix}%
  \crefalias{subsubsection}{subsubappendix}}
\makeatother

\usepackage[disable,textsize=tiny]{todonotes}

\newcommand{\prob}[1]{\Pr\left(#1\right)}

\newcommand{\expect}[1]{\mathds{E}\left[#1\right]}

\newcommand{\indfe}[1]{\ensuremath{\mathds{1}_{#1}}}
\newcommand{\indf}[1]{\indfe{\{#1\}}}

\theoremstyle{plain}
\newtheorem{theorem}{Theorem}[section]

\newtheorem{corollary}[theorem]{Corollary}
\theoremstyle{definition}

\theoremstyle{remark}

\newcommand{\X}{\mathcal{X}}
\newcommand{\Y}{\mathcal{Y}}
\newcommand{\Z}{\mathcal{Z}}
\newcommand{\R}{\mathbb{R}}
\newcommand{\N}{\mathbb{N}}
\newcommand{\F}{\mathcal{F}}

\newcommand{\guard}{\mathtt{guard}}
\newcommand{\countones}{\mathtt{count}}
\newcommand{\eval}{\mathtt{eval}}

\usepackage{mathrsfs}
\newcommand{\filtuc}[1]        {\mathscr{G}_{#1}}
\newcommand{\filt}[1]        {\filtuc{#1}}
\usepackage{algorithm}

\usepackage[noEnd=false]{algpseudocodex} %

\usepackage{xcolor}
\newcommand{\red}[1]{\textcolor{red}{#1}}
\newcommand{\problemname}[1]{\texorpdfstring{\text{{\sc #1}}}{#1}\xspace}
\newcommand{\onemax}{\problemname{OneMax}}

\newcommand{\leadingones}{\problemname{LeadingOnes}}
\newcommand{\jump}{\texorpdfstring{\text{{\sc Jump}}}{Jump}} %
\newcommand{\cycleuncover}{\problemname{CycleUncover}}
\newcommand{\stage}[2]{\textbf{Stage~#1} (\texttt{#2}):}
\newcommand{\vertexcover}{\problemname{VertexCover}}
\newcommand{\bandwidthvc}{\problemname{BandwidthVC}}
\newcommand{\wellcoveredvc}{\problemname{WellCoveredVC}}
\newcommand{\submodmin}{\problemname{SubmodularMin}}

\newcommand{\PK}[1]{
}

\usepackage{listings}
\usepackage{amssymb}
\usepackage{color}
\definecolor{keywordcolor}{rgb}{0.7, 0.1, 0.1}   %
\definecolor{tacticcolor}{rgb}{0.0, 0.1, 0.6}    %
\definecolor{commentcolor}{rgb}{0.4, 0.4, 0.4}   %
\definecolor{symbolcolor}{rgb}{0.0, 0.1, 0.6}    %
\definecolor{sortcolor}{rgb}{0.1, 0.5, 0.1}      %
\definecolor{attributecolor}{rgb}{0.7, 0.1, 0.1} %

\usepackage{pdfpages}
\usepackage{xspace}
\usepackage{newunicodechar}
\newunicodechar{𝒢}{$\mathscr{G}$}
\newunicodechar{ᵐ}{\ensuremath{^{m}}}
\newunicodechar{∏}{\ensuremath{\prod}}
\newunicodechar{𝕜}{\ensuremath{\mathbb{k}}}
\newunicodechar{⋃}{\ensuremath{\cup}}
\newunicodechar{ᶠ}{\ensuremath{^f}}

\title{Agentic Algorithm Design: LLMs that Propose, Formalise, and Prove Runtime Bounds for Evolutionary Algorithms in Lean}

\author{%
  Per Kristian Lehre\\
  School of Computer Science\\
  University of Birmingham, UK\\
  \texttt{p.k.lehre@bham.ac.uk}
}

\begin{document}

\maketitle
\ifarxiv\lhead{Preprint}\fi

\begin{abstract}
LLMs now prove hard theorems in Lean. Can they also \emph{design}
algorithms and prove their correctness and efficiency? LLM-guided
design of optimisation heuristics (such as AlphaEvolve) is purely
empirical. Automated theorem provers with LLMs prove given
statements. So far, there has been little evidence that LLMs can do
both: design an algorithm and guarantee its efficiency by a
mathematical proof.
We present a staged pipeline modelled after research methodology in
design and analysis of evolutionary algorithms. The agent proposes
black-box optimisation algorithms, and Lean guards filter out flawed
designs. The agent then defines a potential function and proves
drift bounds that guarantee an expected runtime improvement over the
baseline algorithm. At no point in the pipeline are the algorithms
compared empirically.
The agent generates a black-box optimisation algorithm with expected
runtime $O(n^2)$ on a \cycleuncover problem, where the classical
(1+1) EA needs runtime $e^{\Omega(n)}$. The pipeline also designs
and analyses a black-box optimisation algorithm for integer-valued
submodular minimisation with expected runtime $O(n^{2+2c})$, and it
achieves the classical runtime $O(n^2)$ on \leadingones. Across 8 problem
classes, 12 of 15 runs with frontier models complete the pipeline
with a machine-checked expected-runtime bound. In 9 of them the bound
is polynomial without relying on offline computation, which the
black-box model does not charge.
An ablation study shows that removing the staging and asking the
agent to zero-shot an algorithmic design with a runtime guarantee
leads to a copy of the seeded baseline algorithm (random search) or
to incomplete proofs. Open-weight models (Goedel-Prover-V2-32B
and Gemma 4) did not complete the pipeline.
Our study suggests that machine-checked specifications in Lean can
replace empirical evaluation in open-ended algorithm discovery. To the best of our knowledge, this is the first example
of an automated pipeline that designs efficient black-box
optimisation algorithms with machine-checked runtime bounds in
Lean.
\end{abstract}

\section{Introduction}

From the perspective of theoretical computer science, algorithm design
must produce two artefacts: First, the \emph{algorithm} itself, a finite
sequence of instructions that when executed in a computational model
transforms within a finite number of steps any problem instance (the
input) into an output (the solution) \citep{KnuthVol1}. Secondly, the
\emph{analysis} of the algorithm, a mathematical proof that the
algorithm produces the correct output on every problem instance, and a
guarantee for the amount of computational resources required by the
algorithm to produce that output (its complexity). In the context of
machine learning, analysis also involves guaranteeing additional
algorithmic properties, such as convergence rates, and regret and/or
generalisation bounds.  Inventing new and better algorithms involves an
iterative developmental process between the two artefacts: a new candidate design drives the analysis in a certain direction, and
obstacles in completion of the analysis can lead to insights into
problem structure which motivate alterations to the algorithm design.
The high cost and deep expertise needed for traditional algorithmic
design have motivated research into automated design of
heuristics. Recent LLM-based heuristic design approaches evaluate a
candidate pool of heuristics on benchmark functions, and
well-performing heuristics are varied through LLM-prompts to produce
better ones. While this has been successful in many domains, the
methodology is highly empirical and does not lead to any guarantees
on the correctness or performance of the heuristics produced.

It is a significant open problem whether AI agents can carry out
algorithmic design as seen from theoretical computer science.  It has
been argued that automated theorem proving (ATP) is challenging for
AI, because proofs are either correct or incorrect, not
in-between. Furthermore, the space of proofs is infinitely large \citep{buzzard_mathematical_2024}. Formal algorithmic design goes
beyond ATP, because it requires both navigating a space of algorithmic designs and automatically proving properties about these.
Yet, the improvement in mathematical capability of frontier AI models
over the last years has been remarkable. Long-standing conjectures
have been closed using publicly available reasoning models with
Lean-checked proof certificates \citep{sothanaphan_resolution_2026}. The time and cost needed to
rigorously prove non-trivial mathematical statements have fallen
significantly.

We take inspiration from the field of runtime analysis of evolutionary
algorithms which has a well-defined approach to design and analysis.
The field requires that algorithms fit in a black-box optimisation
framework \citep{droste_upper_2006} (more often called the query model
in TCS) where optimisation is formulated as an interaction between an
oracle which knows the function to be optimised, and an algorithm
which queries the oracle for fitness values. The field has developed a
small number of highly generic mathematical tools that aid in the
analysis. In particular, we build on drift analysis
\citep{hajek_hitting-time_1982,he_drift_2001} which measures the
progress of the algorithm with a potential function. The potential function, which is introduced for the purpose of the analysis and is
not part of the algorithm itself, maps the state of the algorithm into
a potential value. Bounds on the runtime can be obtained from
properties of the distribution of the drift (change in potential
per oracle query). The black-box model allows for an enormous variety
of algorithms, most of which are useless. To constrain algorithm search,
we use Lean guards to filter out provably flawed designs. We only
allow algorithms that are able to visit all parts of the search space,
that use randomness, and that take into account fitness responses
from the oracle.
The full staged pipeline is shown in \Cref{fig:staged-pipeline}.  The pipeline is seeded with a problem
card that provides the definition of the problem, its problem
instances, and basic properties. Additionally, the pipeline is seeded
with an initial algorithmic design with a stated runtime bound. In all
the experiments described here, we use uniform random search as the initial algorithm, with expected runtime $O(2^n)$. The staged
pipeline consists of some informal stages where an informal model
produces one or more prose descriptions of designs. There are also
formal stages where the agent needs to synthesise a definition or a
proof that is then verified by Lean. In case Lean reports a failure,
the formal stage is repeated, until it backtracks (determined by a
Luby sequence \citep{luby_optimal_1993}), or the formal stage passes. Some formal stages are
gates to filter out flawed designs, others collect definitions that are
used later, e.g., the definition of the algorithm or the potential.
If a run reaches the end of the pipeline, it obtains an
algorithm $A^{i+1}$ with a guaranteed improved runtime compared to
algorithm $A^i$. The pipeline can be looped, e.g., producing an
algorithm $A^{i+2}$ which beats $A^{i+1}$, and so on. In most of the experiments
described here, we make one iteration of the loop; two runs iterate it for three and five
generations.

\paragraph{Contributions.}
To our knowledge, we propose the first automated system for formal
design and analysis of black-box algorithms for pseudo-Boolean
optimisation. The system does not involve any empirical evaluation or
comparison of algorithms. The entire design process is driven by
informal algorithm proposals followed by mathematical formalisation
and analysis in Lean. Our contributions are:

\begin{itemize}
\item We formalise the human process of designing and analysing
  black-box optimisation algorithms through a staged pipeline which
  combines LLMs with Lean.
\item  We demonstrate that an agent using the staged pipeline is able to design efficient
  black-box algorithms with associated machine-checked runtime
  bounds in Lean for a range of problems. On \cycleuncover, the
  constructed algorithm has expected polynomial runtime, whereas the
  classical (1+1) EA has exponential runtime. The agent is able
  to exploit structure in anonymised problem classes.
\item We show through ablation studies that without the staged pipeline,
  the agent fails to construct efficient algorithms.
\end{itemize}

\section{Related work}\label{sec:related-work}

\begin{table}[t]
\centering
\caption{Overview of some related work in automating the design and formal analysis of algorithms. The first group machine-checks analyses of
  \emph{human-designed} algorithms. The second group produces
  machine-checked proofs, but of a theorem or specification
  \emph{given} to it. The third group \emph{discovers} algorithms, but
  validates them only empirically. This work lies at the intersection:
  the algorithm is proposed by an LLM, its analysis is produced by an
  LLM, and the property machine-checked is the expected runtime of a
  stochastic process.}
\label{tab:related-work}
\footnotesize
\setlength{\tabcolsep}{4pt}
\renewcommand{\arraystretch}{1.15}
\begin{tabular}{@{}>{\raggedright\arraybackslash}p{0.17\textwidth}>{\raggedright\arraybackslash}p{0.19\textwidth}>{\raggedright\arraybackslash}p{0.16\textwidth}>{\raggedright\arraybackslash}p{0.09\textwidth}>{\raggedright\arraybackslash}p{0.09\textwidth}>{\raggedright\arraybackslash}p{0.18\textwidth}@{}}
\toprule
\textbf{Work} & \textbf{Algorithmic artefact} & \textbf{Process} & \textbf{Designed by} & \textbf{Analysed by} & \textbf{Machine-checked property} \\
\midrule
\citet{li_formalization_2024}   & first-order convex methods   & deterministic    & human & human & convergence rate \\
\citet{vajjha_certrl_2021}      & value / policy iteration     & stochastic (MDP) & human & human & convergence \\
\citet{schaffeler_formally_2025} & approximate policy iteration & stochastic (MDP) & human & human & convergence \\
\midrule
\citet{ospanov_apollo_2025}; \citet{varambally_hilbert_2025}; \citet{li_goedel-code-prover_2026}
                              & (proof / code search)    & n/a              & n/a   & LLM   & given theorem or code specification \\
\midrule
\citet{romeraparedes_mathematical_2024}; \citet{novikov_alphaevolve_2025}; \citet{lange_shinkaevolve_2025}
                              & heuristics, constructions    & ---              & LLM   & ---   & none (empirical validation) \\
\citet{stein_llamea_2025}       & black-box optimisation algorithms (metaheuristics) & stochastic & LLM & --- & none (empirical benchmarking) \\
\midrule
\textbf{This work}              & black-box optimisation algorithms & \textbf{stochastic} & \textbf{LLM} & \textbf{LLM} & \textbf{correctness \& expected runtime} \\
\bottomrule
\end{tabular}
\end{table}

Research on automating the design and analysis of algorithms can be
classified according to the level of automation in the design and the level of
automation and rigour in the analysis (see \Cref{tab:related-work}).
\emph{Firstly}, there are efforts to formalise existing (informal) mathematical analyses of algorithms manually in proof assistants such as Lean, Coq or Isabelle/HOL. These contributions have the same
strong benefits as formalising any other type of informal mathematics
in Lean, including enabling machine-checked proofs, curation of a
digital library of mathematical results, and providing a foundation
for machine-generated mathematics
\citep{buzzard_mathematical_2024}. Recent efforts in this direction
include \citet{li_formalization_2024} who focused on first-order
convex optimisation methods, with Lean proofs of convergence rates for
some function classes.  They considered deterministic algorithms, and
the proofs therefore avoid technical challenges arising from
stochasticity.  In contrast, there are also Isabelle/HOL
formalisations of complexity analysis of randomised tree data
structures \citep{eberl_verified_2020}.  \citet{kaminski_weakest_2017}
introduced a calculus for formally deriving the expected runtime of
probabilistic programs, such as the coupon collector process. Recent formalisations in machine learning also consider stochastic effects
\citep{vajjha_certrl_2021}.  \citet{schaffeler_formally_2025} analysed
convergence of value and policy iteration for finite-state Markov decision processes.  Other work considers properties of
neural networks, e.g., formalising a classic result about the
expressive power of depth in neural networks
\citep{cohen_expressive,bentkamp_formal_2019}. In contrast to our
work, these works consider pre-existing algorithms and proofs, and
proofs are human-designed.
\emph{Secondly}, the area of LLM-assisted automated theorem proving in Lean
has gained significant momentum recently, with breakthrough results
announced during Summer/Autumn 2026. There are recent examples of
semi-automatic algorithm design and proof generation where humans with
assistance from frontier models design new algorithms with Lean
proofs.
\emph{Thirdly}, there is a large body of research where algorithm design is
automated via LLMs, but where evaluation of candidate algorithms lacks
machine-checked proofs and is entirely empirical
\citep{liu_systematic_2026}.  Systems such as AlphaEvolve
\citep{novikov_alphaevolve_2025}, ShinkaEvolve
\citep{lange_shinkaevolve_2025}, LLaMEA \citep{stein_llamea_2025}, and
EoH \citep{liu_evolution_2024} formulate algorithm design as a search
problem where candidate algorithms are evaluated empirically on a set
of benchmark functions, and an LLM is used to generate algorithm
variants.  These systems do not provide any machine-checked guarantees.  Related approaches in highly restricted problem domains
where correctness can be checked without formal proofs include
AlphaTensor, which optimises matrix multiplication algorithms via
reinforcement learning \citep{fawzi_discovering_2022}.

Apart from automation of algorithmic design and analysis, there are
other relevant works.
An orthogonal task to algorithm design is problem formulation
and modelling. \citet{robbins_flare_2026} proposed FLARE to certify
in Lean the correctness of mixed integer linear program (MILP)
reformulations of combinatorial optimisation problems.
Another related area is automated software verification. Here, the
goal is to synthesise proofs (e.g., in Lean) that a piece of software
satisfies a specification (pre-condition and post-condition) \citep{li_goedel-code-prover_2026}. More ambitiously, vericoding
studies AI systems that take formal specifications as inputs and
produce software along with proofs that the software satisfies the
specification \citep{bursuc_benchmark_2025}.

\section{Runtime analysis of black-box optimisation algorithms}\label{sec:runtime-analysis}

\subsection{Black-box optimisation algorithms}\label{sec:blackbox}

To formally define a design space of optimisation algorithms, we will
adopt the so-called \emph{black-box scenario} from theory of discrete
evolutionary algorithms \citep{droste_upper_2006}, also called query
(or decision tree) complexity in theoretical computer science
\citep{buhrman_complexity_2002}.  This scenario describes a broad
class of algorithms by defining the goal of the algorithm, what
information is available to an algorithm, what actions the algorithm
can take, and how its runtime is measured. The algorithm is given a
class of real-valued fitness functions $F$ over a search space
$\X$. An adversary picks a function $f\in F, f : \X\rightarrow\R$, and
the goal of the algorithm is to find an optimal search point $x^*$
such that $f(x)\leq f(x^*)$ for all $x$ (i.e., maximisation of $f$).
In each iteration $t$, the algorithm queries a search point
$x_t$ based on its previous history, and obtains from the oracle the
function value $f(x_t)$. The search point $x_t$ along with its fitness
$f(x_t)$ is added to the history, and the process repeats until an
optimum is found. The running time of the algorithm is defined as the
number of times the algorithm queries the oracle until the algorithm
queries an optimal search point for the first time. This definition is
based on the assumption that the evaluation of the fitness function is
the most time-consuming part of the algorithm. Following the literature on
runtime analysis of evolutionary algorithms, the algorithm does not specify
a termination criterion. A termination criterion could be specified, but it would not 
impact the runtime of the algorithm as defined below.
The entire interaction between the oracle and the optimisation
algorithm is given in Algorithm \ref{algo:blackbox}. Within the black-box scenario, the algorithm design is fully specified by a tuple
$A=(\Y,\init,\query,\step)$, which will be described below.  Executing
an algorithm $A$ involves executing \Cref{algo:blackbox} with the
given tuple $A$.

\begin{algorithm}[ht]
  \caption{Black-Box Optimisation Algorithm}
  \begin{algorithmic}[1]
    \Require Fitness function $f : \X\rightarrow\R$; random sequence $z_0,z_1, \ldots\in \Z$ with
    $z_t(i)\sim\Unif([0,1])$ i.i.d.\ for all $t,i\in\N$; algorithm $\red{A=(\Y,\init,\query,\step)}$ with
    $\init : \Z\rightarrow \Y$, $\step : \Z\times \Y\times \R\rightarrow \Y$ and $\query : \Y\rightarrow \X$.
    \State Initialise algorithm state $y_0 := \red{\init}(z_0)$
    \For{$t=0$ to $\infty$}
      \State Select search point $x_t := \red{\query}(y_t)$
      \State Query oracle to obtain fitness $f_t := f(x_t)$
      \State Compute the next state $y_{t+1} := \red{\step}(z_{t+1}, y_t, f_t)$
    \EndFor
  \end{algorithmic}
  \label{algo:blackbox}
\end{algorithm}

To allow the algorithm to make random decisions, the algorithm is
provided with a two-way infinite random source $(z_t)_{t\in\N}$, where
for each time step $t\in\N$ and index $i\in\N$, $z_t(i)$ is a
uniformly distributed value in $[0,1]$. We assume an underlying
probability triple $(\Omega,\mathscr{G},\Pr)$ and a filtration
$(\filt{t})_{t\in\N}$ to which $(z_t)_{t\in\N}$ is adapted.  An
algorithm $A$ is associated with a state space $\Y$ where $y_t \in \Y$
corresponds to the memory accessible to the algorithm at iteration
$t$. The initial state is obtained directly using the random source
$\init(z_0)$. In every iteration $t$, the algorithm uses the current
state $y_t$ to decide which search point $x_t$ to query next. The
algorithm transitions to the next state $y_{t+1}$ by taking into
account the next random value $z_{t+1}$, the previous state $y_t$, and
the objective value $f(x_t)$ of the query corresponding to the current
state.

Following \citep{droste_upper_2006}, the running time of algorithm $A$
on $f$ can then be defined by
$T_{A,f} = \min \{t \in \N\mid \forall x\in \X, f(x_t)\geq f(x)\}$.
By construction $(x_t)_{t\in\N}$ and $(y_t)_{t\in\N}$ are adapted to
$(\filt{t})_{t\in\N}$, and $T_{A,f}$ is a stopping time.  The
worst-case expected running time of $A$ over a class of functions $F$
is then
$T_{A,F} = \max_{f\in F} \expect{T_{A,f}}.\label{def:runtime-alg-a}$
All aspects of the above formalisation are indexed by an implicit
problem instance size $n\in\N$. In particular, for each instance size
$n$, we assume a corresponding search space $\X_n$, a fitness function
$f_n : \X_n\rightarrow\mathbb{R}$, an algorithm $A_n$, and a
worst-case expected running time $T_{A_n,F_n}$. We are primarily
interested in the asymptotic growth of $T_{A_n,F_n}$ as a function of
$n$.
Note that this formalisation guarantees that the algorithm is truly
black box: it must be defined independently of the objective
function $f$. Yet, due to the free choice of $\Y$ and $\init$,
$\query$, $\step$, the framework is highly general and can express
most commonly used black-box optimisation algorithms, such as
simulated annealing, population-based evolutionary algorithms, and
estimation of distribution algorithms.

\subsection{Drift analysis}\label{sec:drift}

Drift analysis
\citep{hajek_hitting-time_1982,he_drift_2001} is the main technique used to derive bounds on the runtime $T_A$, and it is
incorporated into our pipeline.  Most evolutionary algorithms are
Markov processes over their state space $\Y$: the state $y_t\in \Y$ of
the algorithm at iteration $t$ completely determines the distribution
of the state $y_{t+1}$ in the next iteration. In principle, Markov
chain theory could be applied to estimate the hitting time of the
process from its state transition matrix. However, the state spaces
induced by the algorithms are often so large that this approach
becomes impractical. Instead, one can introduce a \emph{potential
  function} $\Phi : \Y\rightarrow \R$ which given the current state of
the algorithm measures how far the algorithm is from reaching its
goal, i.e. $Z_t = \Phi(y_t)$. Note that the potential function is not
part of the design of the algorithm, but is only used for the drift
analysis. Designing a good potential function is often the most
challenging part of the analysis of the algorithm. Given
information about the distribution of the change in potential per time
step, i.e., the \emph{drift} $Z_{t+1}-Z_t$ of the process, drift theorems provide statements about the long term properties of the
process. Informally, the \emph{additive drift theorem}
states that if the expected drift is at least
$\varepsilon$, then the expected time to reach distance zero is
$Z_0/\varepsilon$. The proof of the theorem (see \citealp{corus_level-based_2018}) uses standard arguments from martingale
theory; the precise statement is \Cref{lemma:pol-drift} in \Cref{app:drift-proofs}. A Lean
formalisation of the general additive drift theorem will be published in a separate manuscript.
It is proved in the Lean library on which the pipeline builds, not assumed: for every run in which
we printed the axioms (rows 2--5 and 15 of \Cref{tab:runs}), the final runtime theorem depends only
on Lean's standard axioms \texttt{propext}, \texttt{Classical.choice} and \texttt{Quot.sound}.

\Cref{lemma:pol-drift} requires evaluation of an expectation
conditional on a $\sigma$-algebra. To shield the agent from proving
theorems involving measure theory, we introduce the following
simplified theorem which implies condition 1 in
\Cref{lemma:pol-drift}.  Proving drift condition 1 is now reduced to
proving bounds on the probabilities of improving and worsening
steps. To satisfy condition 2 in the additive drift theorem, we assume
the potential is bounded in $[0,1]$. Finally, to prove condition 3,
i.e., integrability of the runtime, we require a lower bound of
jumping directly to the target from every iteration. The set of
conditions exposed in the pipeline is illustrated in
\Cref{fig:drift-regions}. Together, they yield the runtime bound
$4/(\varepsilon p)$ used in the final stage of the pipeline. The proof
of Corollary~\ref{corollary:simple-drift} is given in
\Cref{app:drift-proofs}. Note that the corollary states probabilities
in terms of expectations of indicator random variables, while
the Lean formalisation uses probabilities directly.

\begin{corollary}\label{corollary:simple-drift}
  Assume that for all $t\in\mathbb{N}$, $X_t\in[0,1]$ and
  \begin{align*}
    &\text{(1)}\quad \indf{t<T}\,p\leq\expect{\indf{t<T}\indf{X_{t+1}\leq X_t-\varepsilon}\mid \filt{t}},\\
    &\text{(2)}\quad \expect{\indf{t<T}\indf{X_{t}< X_{t+1}}\mid \filt{t}}\leq \indf{t<T}\,p/2,\\
    &\text{(3)}\quad \expect{\indf{t<T}\indf{X_{t}+\varepsilon< X_{t+1}}\mid \filt{t}}\leq \indf{t<T}\,\varepsilon p/4.
  \end{align*}
  Then $\expect{\indf{t<T}(X_{t+1}-X_t+\varepsilon p/4)\mid \filt{t}}\leq 0$ for all $t\in\mathbb{N}$.
\end{corollary}

\begin{figure}[t]
  \centering
  \includegraphics[width=\linewidth]{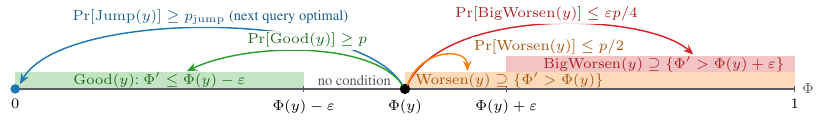}
  \caption{Drift conditions checked by the pipeline (Corollary~\ref{corollary:simple-drift}, Stages
    10f--10i) at a reachable, non-optimal state $y$. The sets of random draws $\Good(y)$, $\Worsen(y)$,
    $\BigWorsen(y)$ and $\Jump(y)$ are shown by where they send the next potential $\Phi'$; the arrows
    give the required probability bounds, which together yield $\expect{T_A}\le 4/(\varepsilon p)$.}
  \label{fig:drift-regions}
\end{figure}

\section{The Staged Pipeline}\label{sec:pipeline}

We now describe the motivation and details of the staged pipeline
shown in \Cref{fig:staged-pipeline}. The task of zero-shotting an
algorithmic design along with its runtime analysis is challenging. We
therefore break down this task into multiple stages according to how a
human researcher in theory of evolutionary algorithms would solve the
problem.  Starting from an initial algorithmic design we want to
improve, we make some informal design proposals. We describe the best
algorithm formally in pseudo-code, then design a potential function
which measures its progress, and finally apply a drift theorem. The
most challenging aspects of this process are to design the algorithm
and choose the potential function. The space of algorithmic designs
allowed by our framework is vast. We therefore propose a staged
pipeline where candidate designs are filtered through formal proof
obligations via proof guards. The formal guards are designed to filter
out flawed designs. The implementation, with its two backends (CLI and
vLLM) and the Kimina Lean server, and the structure of the stage
prompts are described in \Cref{sec:implementation}.

\begin{figure}[t]
  \centering
  \includegraphics[width=\linewidth]{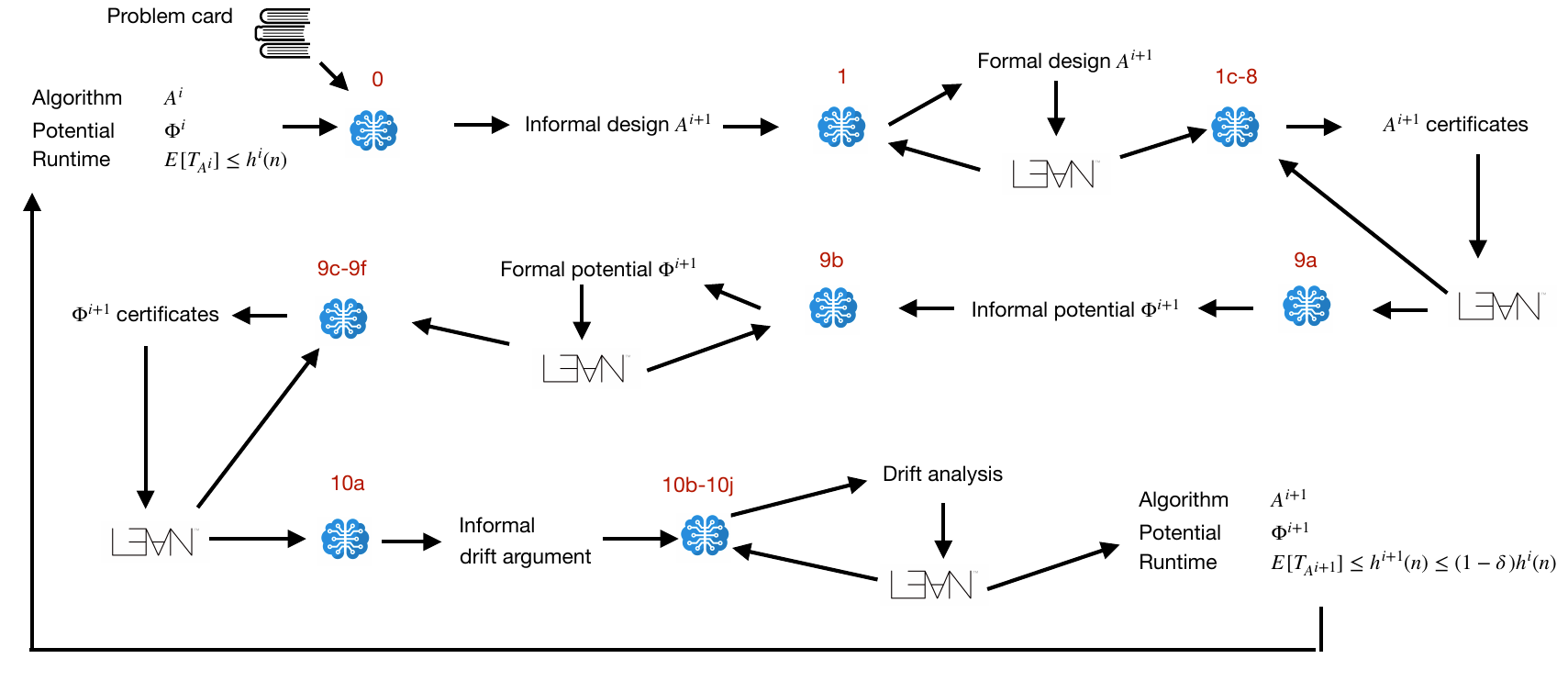}
  \caption{A simplified overview of the staged pipeline. The pipeline is seeded with
    a problem card, an algorithm $A^i$, and a
    proof of its running time. The pipeline produces a new algorithm
    $A^{i+1}$ with improved runtime bound relative to algorithm $A^{i}$.  
    \Cref{tab:pipeline-stages} provides detailed descriptions of
    the pipeline stages shown in red.
    \label{fig:staged-pipeline}
  }
\end{figure}

We now describe the stages in the pipeline informally. The pipeline
has informal stages where the agent proposes in English
an informal design (either of an algorithm or potential function), and
formal stages that are either definitions or formal guards. Formal
statements are stated fully in \Cref{tab:pipeline-stages} in
\Cref{app:pipeline-stages}.

\subsection{Algorithm design (Stages~0--8)}\label{sec:pipeline-algorithm}

Stages~0--8 produce an algorithmic
design along with proofs about its properties.
\stage{0}{DesignLock} We ask the agent to provide an informal
algorithm design along with justification why it would complete the
full pipeline. When using frontier models (rows 1--15 in
\Cref{tab:runs}), the agent is asked to sketch at least one
alternative design before committing to one design in a single
response.  With the open-weight models (rows 19--21 in
\Cref{tab:runs}), the informal model generates eight proposals
sequentially. Each proposal is prompted with all previous ones and
asked to differ from them. A winning proposal is then selected through
a knockout tournament. The proposals were compared in pairs by the
informal model, and the winners advance until one design
remains. Stages~9a and 10a use the same procedure in these runs.
\stage{1}{ConstructAlgorithmComponents} We ask the agent to formalise
in Lean the algorithm as a tuple $A=(\Y,\init,\step,\query)$.
\stage{1c}{StepMeasurability} We require that the step function is a measurable function. This
condition will be required for the drift analysis later in the
pipeline.
\stage{2}{Completeness1} We require that the algorithm must
be able to query any search point in every iteration. More precisely,
for every search point $x\in \X$, and for every reachable state
$y\in \Y$, there exists a random input $z\in \Z$ and fitness value
$f\in \R$, such that if the algorithm receives fitness $f$ for the
query associated with state $y$, it will transition on random input
$z$ to a state where $x$ is queried.  This condition is satisfied by
almost all evolutionary algorithms, e.g., those that involve bit-wise
mutation. Note that the completeness condition can be trivially
satisfied for an arbitrary algorithm by querying a uniform search
point with some non-zero (but perhaps very small probability) in every
step.  The completeness condition is not always necessary for an
efficient algorithmic design. However, the condition makes it
straightforward to prove that the algorithm eventually queries the
optimal search point (i.e., the algorithm is correct), and that for
finite search spaces, the algorithm has a finite expected hitting time
which is required by our formulation of the drift theorem.  The extra
assumption that $y$ is a reachable state and not an arbitrary state
was included to avoid imposing the condition on inconsistent states
that the algorithm never visits. For example, the state of the UMDA
algorithm is a probability vector $p\in[0,1]^n$ where $p_i, i\in[n],$
is the probability of sampling a 1-bit in position $i$. The algorithm
maintains the invariant $p_i\in[1/n,1-1/n]$ in order not to get stuck
in a particular bit-position.
\stage{3}{NonAbsorbing} We require that
the algorithm uses the random input. More precisely, for every
reachable state $y\in \Y,$ and any oracle response $f\in \R$, there
exist two random inputs $z_1,z_2\in \Z$ such that the algorithm in
the next iteration makes different queries on random input $z_1$ and
$z_2$. This condition is not necessary to have an efficient
algorithm. However, the condition filters out many poor algorithms,
such as those that only return a constant value.
\stage{4}{NotRandomSearch} We require that the algorithm can reach a
state where the query distribution is different from the uniform
distribution (random search). Random search has expected running time
$2^{\Omega(n)}$, and any efficient algorithm must have a different
query distribution than random search. We observed that weaker models
applied to the pipeline with random search seeded as algorithm$_0$
tended to produce designs that were essentially random search.  In an
early version of the pipeline, we did not include the reachable state
assumption in the condition. We discovered that this would allow the
model to cheat by devising a state space with unreachable witness
states $\Y$ which could serve as witnesses for the {\tt
  NotRandomSearch} and other properties.
\stage{5}{Adaptiveness} We
require that the state update of the algorithm depends on the fitness
responses from the oracle. Clearly, no algorithm which completely
ignores the oracle can be an effective optimiser on even simple
problem classes such as \onemax.
\stage{6}{QueryNonInjectivity} We
require that there exist two different states which lead to the same
query, i.e., the function $\query$ is not injective. This condition
ensures that the state carries information beyond the current query,
such as the best search point or fitness seen so far, or a learned
parameter, and filters out designs whose state is just the current
search point.  The remaining algorithm
construction stages (Stages~7--8) are designed to package all the definitions
into a Lean structure. This will allow the later proof infrastructure
to process the algorithmic design as one mathematical object. These
stages also include measurability requirements, e.g., that all the
components of the algorithm are measurable functions.  The next main
part is the drift analysis. We would like the agent to prove that the
three conditions of Corollary \ref{corollary:simple-drift} are
satisfied, where $X_t$ corresponds to the potential of the algorithm
at iteration $t$. The pipeline asks the agent to, for every state $y$,
define three subsets of random inputs
$\Good(y),\Worsen(y),\BigWorsen(y)$ corresponding to the three
conditions respectively, drift parameters $\varepsilon$ and $p$, and a
jump probability lower bound $p_{\text{jump}}$.

\subsection{Potential function (Stages~9a--9f)}\label{sec:pipeline-potential}

Stages~9a--9f produce a potential function
used to measure the progress of the algorithm.
\textbf{Stages~9a--9b} (\texttt{PotentialDesign} and
\texttt{DefinePotential}): In the informal Stage~9a, the agent
produces several candidate potential functions together with the
parameters above and an informal ``drift story'' explaining why the
potential function applied to the algorithm has drift towards the
optimum. For experiments with frontier models (Rows 1--15 in
\Cref{tab:runs}), the agent proposes three candidates one after
another, each critiquing the earlier ones, and ends its response with
a ranking. The top-ranked candidate is then formalised in Stage~9b. We
require that the value of the potential is always in the interval
$[0,1]$. Although not required for drift analysis, this requirement
which is not difficult to satisfy through normalisation simplifies the
formal drift analysis machinery later in the pipeline.
\stage{9c}{PotentialOpt} We require that the
potential of a state $y$ is 0 if and only if the algorithm queries an
optimal search point in state $y$.  This is a necessary condition for
the drift theorem to give a meaningful statement about the runtime of
the algorithm.
\stage{9d}{PotentialFineGrained} The seeded design
uses a potential function with only the values 0 and 1, where 0
corresponds to optimal states and 1 all other states. This is
sufficient to prove runtime bound for random search, however it will
be insufficient to capture the gradual progress of an efficient
algorithm. Experiments with weaker models often reproduced the
two-valued potential function. We therefore added the formal guard
that the potential function must contain a value between 0 and 1.
\stage{9e}{PotentialMeasurability} We require that the potential
function is measurable. This property is needed for the later drift
analysis.  The remaining guards in this part of the pipeline involve
proving measurability properties of the potential function, and
packaging into a Lean structure for easy handling later (Stage~9f, \texttt{ConstructPotentialStructure}).

\subsection{Drift certificates and runtime bound (Stages~10a--10j)}\label{sec:pipeline-drift}

Stages~10a--10j prove drift conditions for the potential
and obtain a final runtime bound by applying a drift theorem.
\stage{10f}{PositiveDrift} We
require that for any reachable, non-optimal state $y$, if the
algorithm receives a random input $z\in\Good(y)$, then the potential
decreases by at least $\varepsilon$ in the next iteration.
\stage{10g}{NegativeDriftAndJump} We require that for any non-optimal
reachable state $y$ and any random input $z\in \Z$, if the potential
increases in the next iteration, then $z$ belongs to $\Worsen(y)$, and
if it increases by more than $\varepsilon$, then $z$ belongs to
$\BigWorsen(y)$.  Furthermore, for any reachable state $y$, if the
algorithm receives input $z\in \Jump(y)$, then the potential in the
next iteration is 0. In other words, the algorithm jumps directly to an
optimal search point whenever it receives a random input in
$\Jump(y)$.
\stage{10h}{DriftSetMass} We require that, for every
reachable, non-optimal state $y$, the probabilities that the random
input lies in $\Good(y)$, $\Worsen(y)$ and $\BigWorsen(y)$ satisfy
conditions (1)--(3) of Corollary~\ref{corollary:simple-drift}
(\Cref{fig:drift-regions}).
\stage{10i}{JumpMass} We require that
from all states $y$, $p_{\text{jump}}\leq \prob{z\in\Jump(y)}$. When
satisfied, this condition implies that the runtime of the algorithm is
stochastically dominated by a geometrically distributed random
variable with success probability $p_{\text{jump}}$. The expected
runtime is therefore finite, which implies one of the conditions of
the additive drift theorem. The only requirement on $p_{\text{jump}}$
is that it is positive.
\stage{10j}{RuntimeBound} The final stage is
an application of a runtime bound $\expect{T_A}\leq 4/(\varepsilon p)$
that follows immediately from the previously collected proof
certificates.

\section{Results}
\Cref{tab:runs} summarises all runs; \Cref{fig:completion} in \Cref{app:stage-stats} shows how
each run progressed through the pipeline. Across the completed runs, the drift and runtime stages (10a--10j) take about half of the
wall-clock time, and the incomplete frontier-model runs all stop at a drift stage, most often the
mass bound of Stage~10h (\Cref{app:stage-stats}).
\begin{table}[tbp]
\centering
\caption{All runs. Each row is one pipeline session driven by the named model. Rows 1--18 use the
CLI backend, through which the model exchanges prompts and responses with the pipeline over
stdin/stdout; rows 19--21 use the vLLM backend (\Cref{sec:implementation}). ``Staged'' is the full pipeline of \Cref{tab:pipeline-stages}; ``zero-shot'' is the
one-stage ablation of \Cref{sec:results-zeroshot}. Rows 16--21 are ablations: the zero-shot pipeline, or a weaker model; rows 19--21 use open-weight models run on HPC clusters (\Cref{app:run-open}). \emph{Complete} means that every stage was
machine-checked with no \texttt{sorry}; otherwise the last stage passed is given in parentheses (the run stopped at the following stage). \emph{Order} is the order of the machine-checked bound on $\expect{T_A}$, which holds for all $n\ge3$
and every instance (for staged runs it is $4/(\varepsilon_1p_1)$). \emph{Tokens} is the number of
output tokens (including reasoning) generated by the model that drove the run, from launch to the
last accepted stage; n/r means not recorded. \emph{Lean} is the number of lines of Lean code the
model contributed to the accepted proof, excluding blank lines, comments, the prompt, the API and the
problem card. LO is the fitness of the elitist reference point; the \leadingones card marked
``$e^{-1}$ lemma'' additionally supplies the inequality $e^{-m/(r-1)}\le(1-1/r)^m$. $k$ is the
bandwidth for \bandwidthvc and the gap width for \jump$_k$, and $n^c$ bounds the values of a \submodmin instance.
\emph{\#gen} is the number of generations of the pipeline used to obtain the result; for a run with
several generations, the row reports the last generation only (Tokens and Lean are those of that
generation). Pseudocode, potentials and constants of every run are given in \Cref{app:runs}.}
\label{tab:runs}
\scriptsize
\setlength{\tabcolsep}{3pt}
\renewcommand{\arraystretch}{1.15}
\begin{tabular}{@{}r@{\hspace{5pt}}p{2.2cm}p{1.3cm}p{1.05cm}c@{\hspace{4pt}}>{\raggedright\arraybackslash}p{2.55cm}p{1.1cm}p{1.2cm}rrl@{}}
\toprule
\# & Problem & Model & Pipeline & \#gen & Algorithm & Complete & Order & Tokens & Lean & App.\\
\midrule
1 & \onemax & Opus 5 & staged & 1 & $(1{+}1)$ EA, rate $\frac1{n+1}$ & yes & $O(n^3)$ & n/r & 2{,}036 & \ref{app:run-onemax}\\
2 & \leadingones & Opus 5.5 & staged & 1 & fitness-dependent rate $\frac1{\mathrm{LO}+2}$ & yes & $O(n^3)$ & 160k & 1{,}262 & \ref{app:run-lo}\\
3 & \leadingones, $e^{-1}$ lemma & Opus 5.5$^{\S}$ & staged & 1 & fitness-dependent rate $\frac1{\mathrm{LO}+2}$ & yes & $O(n^2)$ & 191k & 1{,}553 & \ref{app:run-lo2}\\
4 & \jump$_k$ & Opus 5.5 & staged & 1 & $(1{+}1)$ EA, rate $\frac1{n+1}$ & yes & $O(n^{k+2})$ & 144k & 1{,}086 & \ref{app:run-jump}\\
5 & \jump$_k$, $e^{-1}$ lemma & Opus 5.5 & staged & 3 & rate cycling with a learned phase bit & yes & $O(n^{k+1})$ & 396k & 1{,}335 & \ref{app:run-jump-exp}\\
6 & \cycleuncover & Fable 5.1 & staged & 1 & cycle walk (learn-then-solve) & yes & $O(n^2)$ & 538k & 2{,}373 & \ref{app:run-laminate}\\
7 & \wellcoveredvc & Fable 5.1 & staged & 1 & fitness-directed asym.\ mutation & yes & $O(n^4)$ & 366k & 1{,}756 & \ref{app:run-wc}\\
8 & \wellcoveredvc, anonymised card & Fable 5.1 & staged & 1 & polarity-adaptive asym.\ mutation & no (10g) & -- & 275k & 1{,}282 & \ref{app:run-wcvc-fable}\\
9 & \wellcoveredvc, anonymised card & Opus 5.5$^{\ddagger}$ & staged & 1 & regime-switching asym.\ mutation & yes & $O(n^4)$ & 245k & 1{,}530 & \ref{app:run-wcvc-opus}\\
10 & \vertexcover,\newline flat fitness & Fable 5.1 & staged & 1 & edge probing + exact solve & yes & $O(n^3)$ & 400k & 1{,}889 & \ref{app:run-vc}\\
11 & \bandwidthvc,\newline flat fitness & Fable 5.1 & staged & 1 & window probing + exact solve & yes & $O(kn)$ & 460k & 1{,}658 & \ref{app:run-bvc}\\
12 & \bandwidthvc & Fable 5.1 & staged & 1 & direction-adaptive asym.\ mutation & no (10g) & -- & 402k & 1{,}202 & \ref{app:run-bvc2}\\
13 & \bandwidthvc & Fable 5.1 & staged & 1 & fill-then-prune mutation & no (10b) & -- & 192k & 510 & \ref{app:run-bvc-2nd}\\
14 & \bandwidthvc & Opus 5.5 & staged & 1 & window-repair random search & yes & $\Theta(2^n)$ & 283k & 1{,}784 & \ref{app:run-bvc-opus}\\
15 & \submodmin,\newline integer-valued & Opus 5.5 & staged & 5 & Lov\'asz subgradient descent & yes & $O(n^{2+2c})$ & 1.1M & 2{,}901 & \ref{app:run-submod}\\
\midrule
16 & \cycleuncover & Fable 5.1 & zero-shot & 1 & random search & yes & $\Theta(2^n)$ & 41k & 404 & \ref{app:run-laminate-zs}\\
17 & \cycleuncover & Sonnet$^{\dagger}$ & staged & 1 & elitist bit-flip (5 designs) & no (10g) & -- & n/r & 1{,}146 & \ref{app:run-laminate-sonnet}\\
18 & \bandwidthvc & Fable 5.1 & zero-shot & 1 & $(1{+}1)$ EA, rate $\frac1n$ & no & -- & 172k & 112 & \ref{app:run-bvc-zs}\\
19 & \onemax & Goedel-Prover-V2-32B + Gemma~4$^{\P}$ & staged & 1 & per-bit learned bias (PBIL-like) & no (3) & -- & 25.1M & 375 & \ref{app:run-open}\\
20 & \onemax & Goedel-Prover-V2-32B + Gemma~4$^{\P}$ & staged & 1 & per-bit learned bias (PBIL-like) & no (3) & -- & 73.1M & 316 & \ref{app:run-open}\\
21 & \onemax & Goedel-Prover-V2-32B + Gemma~4$^{\P}$ & staged & 1 & random search (fixed biases $\frac12$) & no (9c)$^{\|}$ & -- & 124.3M & 725 & \ref{app:run-open}\\
\bottomrule
\end{tabular}

\smallskip
\raggedright
$^{\dagger}$ Model version not recorded in the artefacts. $^{\ddagger}$ Candidates for the last five
stages were pre-checked directly against the Lean server. $^{\S}$ Candidates for the last three stages were pre-checked in the same way. $^{\P}$ Goedel-Prover-V2-32B for the formal stages (pass@256), Gemma~4 26B-A4B for the prose stages; earlier versions of the staged pipeline with the same 28 stages; tokens are those of all logged samples of the formal model. $^{\|}$ Passed Stage~4 only through a loophole in that pipeline version (see \Cref{app:run-open}). In rows 1--18, all models are Claude models.
\end{table}

\subsection{Theoretical benchmark functions}
Runtime analysis of new evolutionary algorithms is often initiated on
simple theoretical benchmark functions which have a clear structure. Here, we consider the three most common functions, \onemax,
\leadingones, and \jump$_k$, which are defined in \Cref{app:benchmarks}.

The \onemax value of a bitstring is the number of positions in which it differs from
some hidden bitstring $z$. Standard evolutionary algorithms have
expected runtime $O(n\log n)$, which can be seen via a multiplicative
drift argument \citep{doerr_multiplicative_2012}: The probability that
a random bit-flip in the current bitstring improves the fitness is
proportional to the number of bits it has in common with $z$. The
optimal algorithm runs in $\Theta(n/\log n)$ queries by querying the
fitness of $\Theta(n/\log n)$ uniformly sampled bitstrings from which
the hidden bitstring $z$, and the optimum, can be determined with high
probability.
One generation of the staged pipeline rediscovers a variant of the
(1+1) EA with mutation rate $1/(n+1)$ rather than the standard
$1/n$. The true expected runtime of this algorithm is $O(n\log n)$,
however the agent obtains the weaker (but still mathematically
correct) upper bound $O(n^3)$. As potential function, the agent uses
essentially the normalised Hamming-distance to the optimum. The weak
runtime bound is partly due to the pipeline which forces the agent to
use Corollary~\ref{corollary:simple-drift} which does not allow for
multiplicative drift.

The \leadingones value of a bitstring is the number of consecutive
bits it has in common with some hidden bitstring $z$, when the bit
positions are ordered by a hidden permutation $\sigma$.  Typical evolutionary algorithms have runtime $O(n^2)$ which can be seen by an
additive drift argument: a random bit-flip in a selected bitstring
improves its fitness with probability $\Theta(1/n)$, and $\Theta(n)$
improvements are needed to reach the optimum.
One generation of the staged pipeline discovers a variant of (1+1) EA,
but where the mutation rate is chosen as $1/(f(x)+2)$ where $f(x)$ is
the fitness of $x$. This fitness-adaptive mutation rate is nearly
identical to the optimal adaptive mutation rate of $1/(f(x)+1)$
\citep{schaefer_optimal_2010}. For this algorithm, the agent obtains a
sub-optimal runtime of $4(n+1)(n+2)^2=O(n^3)$, while the true runtime
for this algorithm is $(e/4)n^2(1+o(1))$ (Theorem 5 in
\citep{schaefer_optimal_2010}). When inspecting the drift analysis of
the agent, we observed that the missing factor $n$ comes from the
weak lower bound $1/n\leq (1-1/r)^{n-1}$. We therefore created another
problem card which provided the inequality
$1 - 1/r \geq \exp(-1/(r-1))$. The agent now produced a (1+1) EA-type
algorithm with the same fitness-dependent mutation rate and runtime $4e(n+1)(n+2)=O(n^2)$.

The $\jump_k$ function takes the value $k+\onemax$, except for
non-optimal search points at Hamming distance $d<k$ from the optimum
which take the value $n-\onemax$. The problem therefore has a fitness
gap (a local optimum) at distance $k$ from the optimum. Typical
evolutionary algorithms have runtime $O(n^k)$ because they need to
flip $k$ bits simultaneously to jump from the local optimum to the
global optimum. By using more sophisticated diversity and variation
operators, it is possible to reduce the runtime further.
One generation of the staged pipeline discovered a variant of (1+1) EA
with fixed mutation rate $1/(n+1)$. The agent proved a runtime bound
of $O(n^{k+2})$ which is a factor of $n^2$ weaker than the true
runtime for this algorithm which is $O(n^k)$. A factor $n$ was lost
due to using a weak lower bound $p^k(1-p)^n\geq p^{k+1}$ for
$p=1/(n+1)$. In a second set of experiments, we ran the pipeline for
three generations and added the inequality
$1 - 1/r \geq \exp(-1/(r-1))$ to the problem card. This produced an
algorithm with machine-checked runtime $O(n^{k+1})$.

\subsection{\cycleuncover}\label{sec:results-cycle}
The runtime analyses of evolutionary algorithms on \onemax,
\leadingones, and \jump$_k$ are widely known. It is likely that the
LLMs used in this study have seen many of these during their
pre-training (although their Lean formalisations are as far as we know
not available at the time of writing). To rule out the possibility
that the agent simply restates Lean formalisations of well-known
algorithms and analyses, we would like to test the agent on a new
problem not previously discussed in the literature. We therefore
designed \cycleuncover (defined in \Cref{fig:cycleuncover}), which is a problem solvable in polynomial time,
but not by a typical evolutionary algorithm. The problem is
independent set on a cycle graph, where the length of the cycle is
even by construction. The fitness function is defined such that
infeasible solutions lead to a deceptive trap at large Hamming
distance from the optimum. Given the hidden permutation $\sigma$, a
solution is simply to choose every second vertex along the permutation
cycle.

\begin{figure}[t]
  \centering
  \begin{minipage}[c]{0.46\textwidth}
    \centering
    \includegraphics[width=\linewidth]{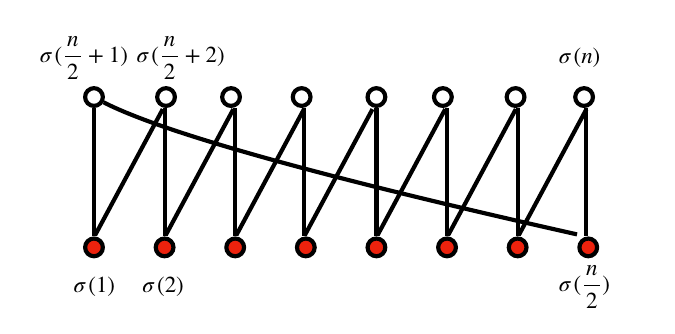}
  \end{minipage}\hfill
  \begin{minipage}[c]{0.52\textwidth}
      Given a permutation  $\sigma : [n]\rightarrow [n]$, let
      $h := \lfloor n/2\rfloor$ and define
      \begin{align*}
        \countones(x) & := \sum_{j=1}^{h} \left(x_{\sigma(j)} + x_{\sigma(h+j)}\right)\\
        \guard(x) & := \forall j\in[h],\;
                    x_{\sigma(j)} + x_{\sigma(h+j)}\leq 1 \\
                  & \qquad \wedge\;
                    x_{\sigma(j)} + x_{\sigma(h+(j \bmod h)+1)}\leq 1 \\
        \eval(x)  & :=
                    \begin{cases}
                      2h + \countones(x) & \text{ if } \guard(x)\\
                      \countones(x)     & \text{ otherwise }
                    \end{cases}
      \end{align*}
  \end{minipage}
  \caption{\cycleuncover. Right: problem definition in human-readable
    form. Left: visualisation of an instance for even $n$ (not shown
    to the agent). The vertices $\sigma(1),\dots,\sigma(2h)$ form a
    cycle of even length $2h$; for odd $n$, bit $\sigma(n)$ does not
    affect the fitness. The problem is to find a maximum independent
    set of the cycle; an optimal selection, of value $3h$, is shown
    in red.}
  \label{fig:cycleuncover}
\end{figure}

The following theorem states that the classical (1+1) EA has expected
exponential runtime on this problem (see \Cref{app:cycle-lb} for a
proof).
\begin{theorem}\label{thm:oneplusone-cycleuncover}
  The expected runtime of the (1+1) EA on \cycleuncover is $e^{\Omega(n)}$.
\end{theorem}

The agent produced an algorithm for \cycleuncover with a
machine-checked expected runtime of $O(n^2)$ (row 6 of
\Cref{tab:runs}).  The algorithm works differently from an
evolutionary algorithm: it builds up a two-colouring of the graph, and
then queries the optimum. The two-colouring is obtained by keeping a
current vertex $u$ which is coloured, and then by repeatedly sampling
uniformly at random a vertex $v$ until $v$ is both uncoloured and
adjacent to $u$. This newly found vertex is coloured oppositely to $u$
and becomes the new current vertex. To determine whether two vertices
$u$ and $v$ are adjacent in the graph, the algorithm queries the
fitness of the bitstring where only $u$ and $v$ are selected. As shown
in \Cref{fig:cycleuncover}, the fitness is 2 if and only if $u$ and
$v$ are connected.  After the first step, $u$ has at most one
uncoloured neighbour, and the probability of sampling one is
$\Omega(1/n)$. The algorithm therefore needs $O(n)$ queries in
expectation to obtain the next vertex $v$.  Repeating this procedure
$O(n)$ times, the algorithm obtains a two-colouring after $O(n^2)$
queries in expectation.  From the complete two-colouring, the
algorithm can queries the optimum by choosing only vertices with the
same colour. Since the cycle has even length, vertices of either
colour class is optimal.  This result suggests that, even for a
problem that to our knowledge has not been studied in this form, the
pipeline can identify the problem structure and provide a tailored and
efficient algorithm.

\subsection{\vertexcover}
We now turn to instance classes of the NP-hard minimum \vertexcover
problem. Given an undirected graph $G=(V,E)$, the minimum vertex cover
problem asks for the smallest subset $V'\subseteq V$ of the vertices
such that for each edge $\{u,v\}\in E$, at least one of the vertices
$u$ and $v$ belong to $V'$. This problem is NP-hard, but a simple
greedy algorithm returns a 2-approximate solution in polynomial time.
We do not attempt to design approximation algorithms here, we focus on
finding a globally optimal vertex cover.
We adopt the conventional pseudo-Boolean representation of
\vertexcover instances.  Given a graph with $n$ vertices
$V=\{1,2,\ldots,n\}$, we encode a subset of the vertices $V'$ by a
bitstring $x \in \{0,1\}^n$, where $x_i=1$ if and only if vertex
$i\in V'$. We adapt the multi-objective formulation from
\citep{kratsch_fixed-parameter_2013} into a single-objective fitness
function which counts the number of chosen vertices (first term),
and the number of uncovered edges (second term).
\begin{align}
  f(x) = -\sum_{i=1}^n x_i - (n+1)\sum_{\{i,j\}\in E} (1-x_i)(1-x_j).\label{eq:vertex-cover}
\end{align}
Classical EAs applied naively to \vertexcover perform
poorly.  The (1+1) EA and variants thereof need at least exponential
time to achieve an approximation factor better than
$(1-\varepsilon)/\varepsilon$ for
$n^{\delta-1}\leq \varepsilon\leq 2$, i.e., almost arbitrarily bad
approximation ratio
\citep{friedrich_approximating_2010,oliveto_analysis_2008}. In
contrast, the trivial greedy algorithm obtains a 2-approximation in
polynomial time.
Using more sophisticated EAs and adding more problem-specific
information can improve the
performance. \citet{kratsch_fixed-parameter_2013} used the number of
selected vertices as first objective, and the number of uncovered
edges as second objective. With this formulation, a multi-objective EA
obtains an optimal vertex cover in expected time
$O(OPT\cdot n^4+n\cdot 2^{OPT^2+OPT})$, where $OPT$ refers to the size
of the smallest vertex cover. Using the ILP-relaxation as second
objective leads to a lower runtime bound. Assuming only problem
instances where $OPT\leq k$ for some fixed parameter $k$ yields a
polynomial runtime, i.e., the algorithm is a fixed-parameter
algorithm. OPT cannot be bounded by a constant in general instances,
and the runtime bound is therefore exponential on the entire problem
class.

We would like to consider non-trivial combinatorial optimisation
problems where the agent cannot rely on published results from the
theoretical literature on black-box optimisation algorithms, or on
formalisations in Mathlib. We also want to consider problems which we
expect to be in $P$. We therefore chose the problem of computing a
minimum vertex cover in well-covered graphs, which we call
\wellcoveredvc.  Recall that a vertex cover is \emph{minimum} if the
graph admits no vertex cover with fewer vertices.  A vertex
cover is called \emph{minimal} (not to be confused with minimum) if no strict subset of $V'$ is a vertex cover. By definition, a
minimum vertex cover is always a minimal vertex cover, but the
converse is not true for all graphs.  Graphs where the minimum and
minimal vertex covers coincide are called \emph{well-covered} graphs.

The fitness function in Eq. \eqref{eq:vertex-cover} combined with the
choice of graph class gives the problem several structural properties
that are helpful for search heuristics, and which are provided to the
agent in the problem card.  Given two search points with different
number of uncovered edges, the one with fewer uncovered edges always
has higher fitness (exposed as lemma {\tt
  wc\_F\_lt\_of\_uncovered\_lt} in the problem card).  Also, if
$\{i,j\}\in E$ is an uncovered edge in $x$, then flipping bit
$x_i$ reduces the number of uncovered edges (exposed as lemma {\tt
  wc\_uncovered\_flip\_lt} in the problem card).
It is also straightforward to see that $f$ contains no local
optima. Consider any non-optimal search point $x$. If $x$ is not a vertex cover, then there must exist an uncovered edge $\{i,j\}$, and
a search point $x'$ obtained by flipping bit $i$ or bit $j$ in $x$ has strictly
higher fitness than $x$. Assume the other case, that $x$ is a vertex
cover but not a minimum vertex cover. Then since the graph is
well-covered, $x$ is also not a minimal vertex cover. Therefore, there
exists a vertex $i$ such that $x_i=1$ and flipping bit $i$ in $x$
produces a strictly fitter vertex cover $x'$.  The fitness therefore
increases.

In a first set of experiments, we provided a problem card containing
all the structural information described above. In this case, the
agent produced an optimisation algorithm similar to a (1+1) EA, but
with regime-changing asymmetric mutation rates. As long as the current
search point is not a vertex cover (revealed by the fitness being less
than $-n$), the mutation-rates were chosen asymmetrically so that
flipping 0-bits into 1-bits was twice as likely as the reverse.  Once
the current search point represented a vertex cover, the mutation-rate
was chosen asymmetrically in the opposite direction, making bit-flips
from 1 into 0 twice as likely as bit-flips from 0 into 1. The overall
runtime obtained was $O(n^4)$. The problem card in the first set of
experiments revealed all the structural information needed to design
an efficient algorithm. In a second set of experiments, we created an
anonymised problem card which does not mention the terms vertex cover
or well-covered graphs, only providing raw fitness definitions. An
experiment with Fable 5.1 failed to obtain an algorithm in this
setting. However, Opus 5.5 successfully generated a (1+1) EA with
regime-changing mutation rates. This time, the asymmetry between the
rates was more pronounced, with the high mutation rate a factor
$\Theta(n)$ larger than the small rate. Overall, the obtained
runtime was $O(n^4)$.

The problem \bandwidthvc is
\vertexcover restricted to graphs with bandwidth at most $k$. A graph
$G=(V,E)$ with vertex set $V=\{0, 1, \ldots n-1 \}$ is said to have
\emph{bandwidth} $k$ if for all edges $\{u,v\}\in E$, it holds
$|u-v|\leq k$. A low bandwidth $k$ implies that edges are only allowed
between ``nearby'' vertices with respect to their numbering. In
particular, $k=1$ allows only subgraphs of the path $0,1,\ldots,n-1$, while $k=n-1$ allows
every graph.

Since \vertexcover is NP-hard, we did not expect the agent to produce a
good algorithm for general graphs. With a flat fitness function, which
only reveals whether a search point is a cover and its number of
1-bits (\Cref{app:runs}), the agent instead came up with an algorithm
that probes all the $\Theta(n^2)$ potential edges in the graph to learn
the problem instance, then without making further queries, computes
offline the optimum for that instance, and finally queries the optimum
(row 10, \Cref{app:run-vc}). The overall runtime bound was $O(n^3)$.
This demonstrates that the agent is able to ``cheat'' this evaluation
framework by exploiting a well-known weakness in how algorithms are
evaluated in the black-box scenario: Algorithms are only charged for
making fitness queries, and not for offline computation. As the agent
shows, the black-box complexity of NP-hard problems is polynomial.
On \bandwidthvc with the same flat fitness function (row 11,
\Cref{app:run-bvc}), the agent used the same strategy but exploited the
bandwidth: it probes only the vertex pairs within distance $k+1$,
giving a runtime bound of $O(kn)$. In both runs, the minimum cover is
defined non-constructively in Lean, so the computation between queries
is not bounded.

With the graded fitness function~\eqref{eq:vertex-cover}, \bandwidthvc
proved harder (rows 12--14). Two runs with Fable 5.1 designed
asymmetric mutation operators but did not complete the pipeline: one
stopped at the mass bound of Stage~10h, which is false for its design
(\Cref{app:run-bvc2}), and the other at Stage~10c
(\Cref{app:run-bvc-2nd}). The run with Opus 5.5 completed with a
window-repair random search (\Cref{app:run-bvc-opus}), whose bound
$\Theta(2^n)$ improves on the random search baseline only by a constant
factor.

\subsection{Sub-modular minimisation}
We now turn to the problem of minimising sub-modular functions (SFM),
a problem with important applications in machine learning.  We
represent any subset $S\subseteq [n]$ as a bitstring $x\in \{0,1\}^n$
in the natural way by setting $x_i=1$ if and only if $i\in S$. We now
call a pseudo-Boolean function $f : \{0,1\}^n\rightarrow\R$ submodular
if and only if for all subsets $S\subseteq T\subseteq [n]$ and
$v\not\in T$, it holds $f(S\cup\{v\})-f(S)\geq f(T\cup \{v\})-f(T)$.
We will assume that $f$ is integer-valued with range in the interval
$[0,B]$ where $B:=n^c$ for some unspecified constant $c$. There exist theoretical algorithms that minimise SFMs in
$O(n^3\log\log(n)/\log(n))$ oracle queries
\citep{jiang_minimizing_2023}, i.e., independent of $B$. However, the
more practically useful Fujishige-Wolfe algorithm has running time
$O(n^2B^2)$. We are not aware of runtime analyses of evolutionary
algorithms for general SFM minimisation, but results for special cases
such as the minimum spanning tree problem are available
\citep{neumann_randomized_2007}.

We expected this problem to be significantly more challenging than the previous problems.  We therefore added to the problem card the advice that this is a challenging problem where a naive elitist hill-climber
may not work, and where the agent needs to carefully design the
potential function and exploit submodularity. Furthermore, we ran
several generations of the pipeline. The algorithmic design and
analysis produced in a first generation of the pipeline was seeded
into the following generation, etc.
Through all generations, the agent produced the same class of
algorithm, which can be described as projected subgradient descent on
the Lov\'asz extension.  The first generation returned a runtime bound
of $12nB^2(2n+4)(n+1)=O(n^{3+2c})$. The second generation improved the
constant factor, reducing the runtime to $16n^2B^2(n+2)=O(n^{3+2c})$. The third
generation saw a factor-$n$ speed-up, with a runtime of
$32nB^2(n+1)=O(n^2B^2)$. Note that this is asymptotically the same as
the Fujishige-Wolfe algorithm. Generation four and five only brought
constant-factor improvements, leading to the approximate upper bound
$8.13\cdot B^2(n+1)^2$.

\subsection{Ablation experiments}\label{sec:results-zeroshot}

To evaluate the impact of the staged pipeline, we developed an
unstaged pipeline consisting of a single stage: Given a problem card,
the agent is prompted to construct an algorithm for the problem which
type-checks with the black-box framework, and which has better runtime
than the seeded algorithm. The agent is not given the drift analysis
API, the measurability toolkit, the proof cookbook, and most random
stream properties. In order to be able to verify the algorithmic
construction in Lean, we provided the agent with Lean-definitions of
the black-box optimisation algorithm framework, and definitions of
runtime. In order for the algorithm to be able to use the random
source, we also provided basic clauses related to sampling from the
random stream. Since the ablation only concerned the staged pipeline,
we kept the problem cards unchanged. We also provided the agent with
the definition of the seeded algorithm (random search).

For \cycleuncover, Fable 5.1 reproduced the trivial random search
algorithm (i.e., it copies the provided {\tt algorithm0} with minor
syntactic changes). Note that this failure mode would have been caught
by the {\tt NotRandomSearch} guard in the staged pipeline. Instead of
relying on drift analysis, the agent applied the tail-sum formula
$\expect{T}=\sum_{k=1}^\infty \prob{k \leq T}$ to the runtime variable
$T$. The agent recognised that the problem involved a hidden
permutation and two global optima, giving a runtime bound of $2^n/2$.

For \bandwidthvc, Fable 5.1 constructed the (1+1) EA, and proved all
formalisation properties. However, the model failed to produce a
runtime bound. The model attempted to apply a fitness-level
argument. This is one of the simplest-to-use techniques for analysing
the runtime of the (1+1) EA introduced more than two decades ago
\citep{wegener_methods_2002}, but it has limited applicability to
more complex algorithms and fitness landscapes.

We then investigated weaker models. Claude Sonnet (row 17) passed Stages 0--10g on
\cycleuncover, but needed between 54 and about 1{,}000 attempts per stage, and it did not pass the
mass bound of Stage~10h, because its Good set is empty at some reachable states
(\Cref{app:run-laminate-sonnet}). The open-weight models, Goedel-Prover-V2-32B for the formal stages
and Gemma~4 26B-A4B for the prose stages (vLLM backend, 256 samples per attempt), did not complete
the pipeline on \onemax in any of 14 recent runs (\Cref{app:run-open}). Rows 19--21 of
\Cref{tab:runs} are representative. Two runs stop at \texttt{NotRandomSearch} (Stage~4) after 13 and
24 batches of 256 samples, although the statement is true for their PBIL-like designs. The third
passes Stage~4 only through a loophole of an earlier pipeline version, since its algorithm is
random search, and then stops at \texttt{PotentialFineGrained} (Stage~9d), again on a true statement.
Frontier models stop, if at all, at a drift stage, whereas the open-weight models stop much earlier
(\Cref{app:stage-stats}), and they generate 25--124M output tokens per run, about two orders of
magnitude more than the frontier-model runs.

\section{Conclusion}

We have presented a staged pipeline in which an agent designs black-box optimisation algorithms and
proves their expected runtime in Lean, with no empirical evaluation. Across 8 problem classes, 12 of
15 runs with frontier models completed the pipeline. Nine of these give a polynomial bound for an
algorithm whose computation between queries is explicit; of the other three, two learn a
\vertexcover instance with polynomially many queries and then solve it offline, at no cost in the
black-box model, and one improves on random search only by a constant factor. The agent designed an $O(n^2)$ algorithm for
\cycleuncover, on which the (1+1) EA needs $e^{\Omega(n)}$, matched the $O(n^2)$ runtime of classical evolutionary algorithms on
\leadingones, and rediscovered projected subgradient descent on the Lov\'asz extension for
integer-valued submodular minimisation, improving its own bound to $O(n^{2+2c})$ over five
generations. Without the staging, or with open-weight models, the pipeline did not produce efficient
algorithms. The bounds are not yet state of the art; richer drift theorems, more iterations of the
design loop, and continuous search spaces are natural next steps.

\paragraph{Limitations.} Only frontier models completed the
pipeline. A more complete experimental design allowing for a
statistical analysis was not feasible due to the high cost of the
runs.  The agent was only provided with a simplified drift theorem
(Corollary~\ref{corollary:simple-drift}). With standard potential
functions, this theorem makes more precise variable drift analyses
difficult, leading to non-tight runtime bounds such as $O(n^3)$ on
\onemax for (1+1) EA variants. Furthermore, the search space is
restricted to bitstrings.

\newpage

\section*{Acknowledgments}
This research was partly supported by a Turing AI Fellowship (EPSRC
grant ref EP/V025562/1).  Calculations were performed using the
Sulis Tier 2 HPC platform hosted by the Scientific Computing
Research Technology Platform at the University of Warwick. Sulis is
funded by EPSRC Grant EP/T022108/1 and the HPC Midlands+
consortium. The author also acknowledges the use of resources
provided by the Isambard-AI National AI Research Resource
(AIRR). Isambard-AI is operated by the University of Bristol and is
funded by the UK Government’s Department for Science, Innovation and
Technology (DSIT) via UK Research and Innovation; and the Science
and Technology Facilities Council [ST/AIRR/I-A-I/1023].

\section*{AI Use Statement}
This paper addresses the scientific question of whether large language models can generate
algorithms together with proofs of their efficiency, so LLMs were used extensively as the object of
study. All synthesised algorithms, potential functions, drift certificates and their Lean proofs
reported in \Cref{tab:runs} were generated by the models named there (Claude Opus~5, Opus~5.5,
Fable~5.1 and Sonnet; Goedel-Prover-V2-32B and Gemma~4), and every proof was checked by Lean. In
addition, LLMs were used to help derive the combinatorial drift part of the Lean library used by the
pipeline (\Cref{sec:runtime-analysis}), whose results are likewise checked by Lean; to translate the Lean definitions of the synthesised algorithms into the human-readable
pseudo-code of \Cref{app:runs}, and to collect the potentials and constants reported there, all of
which we checked against the Lean artefacts; and to polish the writing. The author takes full responsibility for the
content of the paper.

\bibliographystyle{iclr2027_conference}
\bibliography{references}

\begin{thebibliography}{43}
\providecommand{\natexlab}[1]{#1}
\providecommand{\url}[1]{\texttt{#1}}
\expandafter\ifx\csname urlstyle\endcsname\relax
  \providecommand{\doi}[1]{doi: #1}\else
  \providecommand{\doi}{doi: \begingroup \urlstyle{rm}\Url}\fi

\bibitem[Baluja(1994)]{pbil1994}
Shumeet Baluja.
\newblock Population-{Based} {Incremental} {Learning}: {A} {Method} for
  {Integrating} {Genetic} {Search} {Based} {Function} {Optimization} and
  {Competitive} {Learning}.
\newblock Technical Report CMU-CS-94-163, School of Computer Science, Carnegie
  Mellon University, Pittsburgh, PA, June 1994.
\newblock URL \url{https://apps.dtic.mil/sti/pdfs/ADA282654.pdf}.

\bibitem[Bentkamp et~al.(2019)Bentkamp, Blanchette, and
  Klakow]{bentkamp_formal_2019}
Alexander Bentkamp, Jasmin~Christian Blanchette, and Dietrich Klakow.
\newblock A {Formal} {Proof} of the {Expressiveness} of {Deep} {Learning}.
\newblock \emph{Journal of Automated Reasoning}, 63\penalty0 (2):\penalty0
  347--368, August 2019.
\newblock ISSN 0168-7433, 1573-0670.
\newblock \doi{10.1007/s10817-018-9481-5}.
\newblock URL \url{http://link.springer.com/10.1007/s10817-018-9481-5}.

\bibitem[Buhrman \& de~Wolf(2002)Buhrman and de~Wolf]{buhrman_complexity_2002}
Harry Buhrman and Ronald de~Wolf.
\newblock Complexity measures and decision tree complexity: a survey.
\newblock \emph{Theoretical Computer Science}, 288\penalty0 (1):\penalty0
  21--43, October 2002.
\newblock ISSN 0304-3975.
\newblock \doi{10.1016/S0304-3975(01)00144-X}.
\newblock URL
  \url{https://www.sciencedirect.com/science/article/pii/S030439750100144X}.

\bibitem[Bursuc et~al.(2025)Bursuc, Ehrenborg, Lin, Astefanoaei, Chiosa,
  Kukovec, Singh, Butterley, Bizid, Dougherty, Zhao, Tan, and
  Tegmark]{bursuc_benchmark_2025}
Sergiu Bursuc, Theodore Ehrenborg, Shaowei Lin, Lacramioara Astefanoaei,
  Ionel~Emilian Chiosa, Jure Kukovec, Alok Singh, Oliver Butterley, Adem Bizid,
  Quinn Dougherty, Miranda Zhao, Max Tan, and Max Tegmark.
\newblock A benchmark for vericoding: formally verified program synthesis,
  September 2025.
\newblock URL \url{https://arxiv.org/abs/2509.22908v1}.

\bibitem[Buzzard(2024)]{buzzard_mathematical_2024}
Kevin Buzzard.
\newblock Mathematical reasoning and the computer.
\newblock \emph{Bulletin of the American Mathematical Society}, 61\penalty0
  (2):\penalty0 211--224, February 2024.
\newblock ISSN 1088-9485, 0273-0979.
\newblock \doi{10.1090/bull/1833}.
\newblock URL
  \url{https://www.ams.org/bull/2024-61-02/S0273-0979-2024-01833-3/}.

\bibitem[Böttcher et~al.(2010)Böttcher, Doerr, and
  Neumann]{schaefer_optimal_2010}
Süntje Böttcher, Benjamin Doerr, and Frank Neumann.
\newblock Optimal {Fixed} and {Adaptive} {Mutation} {Rates} for the
  {LeadingOnes} {Problem}.
\newblock In Robert Schaefer, Carlos Cotta, Joanna Kołodziej, and Günter
  Rudolph (eds.), \emph{Parallel {Problem} {Solving} from {Nature}, {PPSN}
  {XI}}, pp.\  1--10. Springer Berlin Heidelberg, Berlin, Heidelberg, 2010.
\newblock ISBN 978-3-642-15843-8 978-3-642-15844-5.
\newblock \doi{10.1007/978-3-642-15844-5_1}.
\newblock URL \url{http://link.springer.com/10.1007/978-3-642-15844-5_1}.

\bibitem[Cohen et~al.(2016)Cohen, Sharir, and Shashua]{cohen_expressive}
Nadav Cohen, Or~Sharir, and Amnon Shashua.
\newblock On the {Expressive} {Power} of {Deep} {Learning}: {A} {Tensor}
  {Analysis}.
\newblock \emph{Journal of Machine Learning Research}, 49:\penalty0 1--31,
  2016.

\bibitem[Corus et~al.(2018)Corus, Dang, Eremeev, and
  Lehre]{corus_level-based_2018}
Dogan Corus, Duc-Cuong Dang, Anton~V. Eremeev, and Per~Kristian Lehre.
\newblock Level-{Based} {Analysis} of {Genetic} {Algorithms} and {Other}
  {Search} {Processes}.
\newblock \emph{IEEE Transactions on Evolutionary Computation}, 22\penalty0
  (5):\penalty0 707--719, October 2018.
\newblock ISSN 1089-778X, 1089-778X, 1941-0026.
\newblock \doi{10.1109/TEVC.2017.2753538}.
\newblock URL \url{https://ieeexplore.ieee.org/document/8039236/}.

\bibitem[Doerr et~al.(2012)Doerr, Johannsen, and
  Winzen]{doerr_multiplicative_2012}
Benjamin Doerr, Daniel Johannsen, and Carola Winzen.
\newblock Multiplicative {Drift} {Analysis}.
\newblock \emph{Algorithmica}, 64\penalty0 (4):\penalty0 673--697, December
  2012.
\newblock ISSN 1432-0541.
\newblock \doi{10.1007/s00453-012-9622-x}.
\newblock URL \url{https://doi.org/10.1007/s00453-012-9622-x}.

\bibitem[Dos~Santos et~al.(2025)Dos~Santos, de~Saxcé, Wang, Wang, Baksys,
  Unsal, Liu, Liu, and Li]{dos_santos_kimina_2025}
Marco Dos~Santos, Hugues de~Saxcé, Haiming Wang, Ran Wang, Mantas Baksys, Mert
  Unsal, Junqi Liu, Zhengying Liu, and Jia Li.
\newblock Kimina {Lean} {Server}: A high-performance {Lean} server for
  large-scale verification.
\newblock arXiv:2504.21230, 2025.

\bibitem[Droste et~al.(2002)Droste, Jansen, and Wegener]{droste_analysis_2002}
Stefan Droste, Thomas Jansen, and Ingo Wegener.
\newblock On the analysis of the (1+1) evolutionary algorithm.
\newblock \emph{Theoretical Computer Science}, 276\penalty0 (1--2):\penalty0
  51--81, 2002.

\bibitem[Droste et~al.(2006)Droste, Jansen, and Wegener]{droste_upper_2006}
Stefan Droste, Thomas Jansen, and Ingo Wegener.
\newblock Upper and {Lower} {Bounds} for {Randomized} {Search} {Heuristics} in
  {Black}-{Box} {Optimization}.
\newblock \emph{Theory of Computing Systems}, 39\penalty0 (4):\penalty0
  525--544, July 2006.
\newblock ISSN 1433-0490.
\newblock \doi{10.1007/s00224-004-1177-z}.
\newblock URL \url{https://doi.org/10.1007/s00224-004-1177-z}.

\bibitem[Eberl et~al.(2020)Eberl, Haslbeck, and Nipkow]{eberl_verified_2020}
Manuel Eberl, Max~W. Haslbeck, and Tobias Nipkow.
\newblock Verified {Analysis} of {Random} {Binary} {Tree} {Structures}.
\newblock \emph{Journal of Automated Reasoning}, 64\penalty0 (5):\penalty0
  879--910, June 2020.
\newblock ISSN 1573-0670.
\newblock \doi{10.1007/s10817-020-09545-0}.
\newblock URL \url{https://doi.org/10.1007/s10817-020-09545-0}.

\bibitem[Fawzi et~al.(2022)Fawzi, Balog, Huang, Hubert, Romera-Paredes,
  Barekatain, Novikov, R.~Ruiz, Schrittwieser, Swirszcz, Silver, Hassabis, and
  Kohli]{fawzi_discovering_2022}
Alhussein Fawzi, Matej Balog, Aja Huang, Thomas Hubert, Bernardino
  Romera-Paredes, Mohammadamin Barekatain, Alexander Novikov, Francisco~J.
  R.~Ruiz, Julian Schrittwieser, Grzegorz Swirszcz, David Silver, Demis
  Hassabis, and Pushmeet Kohli.
\newblock Discovering faster matrix multiplication algorithms with
  reinforcement learning.
\newblock \emph{Nature}, 610\penalty0 (7930):\penalty0 47--53, October 2022.
\newblock ISSN 1476-4687.
\newblock \doi{10.1038/s41586-022-05172-4}.
\newblock URL \url{https://www.nature.com/articles/s41586-022-05172-4}.

\bibitem[Friedrich et~al.(2010)Friedrich, He, Hebbinghaus, Neumann, and
  Witt]{friedrich_approximating_2010}
Tobias Friedrich, Jun He, Nils Hebbinghaus, Frank Neumann, and Carsten Witt.
\newblock Approximating {Covering} {Problems} by {Randomized} {Search}
  {Heuristics} {Using} {Multi}-{Objective} {Models}.
\newblock \emph{Evolutionary Computation}, 18\penalty0 (4):\penalty0 617--633,
  December 2010.
\newblock ISSN 1063-6560.
\newblock \doi{10.1162/EVCO_a_00003}.
\newblock URL \url{https://doi.org/10.1162/EVCO_a_00003}.

\bibitem[{Gemma Team}(2026)]{gemma_team_gemma_2026}
{Gemma Team}.
\newblock Gemma 4 {Technical} {Report}, July 2026.
\newblock URL \url{http://arxiv.org/abs/2607.02770}.
\newblock arXiv:2607.02770 [cs.CL].

\bibitem[Hajek(1982)]{hajek_hitting-time_1982}
Bruce Hajek.
\newblock Hitting-{Time} and {Occupation}-{Time} {Bounds} {Implied} by {Drift}
  {Analysis} with {Applications}.
\newblock \emph{Advances in Applied Probability}, 14\penalty0 (3):\penalty0
  502--525, 1982.
\newblock ISSN 0001-8678.
\newblock \doi{10.2307/1426671}.
\newblock URL \url{https://www.jstor.org/stable/1426671}.
\newblock Publisher: Applied Probability Trust.

\bibitem[He \& Yao(2001)He and Yao]{he_drift_2001}
Jun He and Xin Yao.
\newblock Drift analysis and average time complexity of evolutionary
  algorithms.
\newblock \emph{Artificial Intelligence}, 127\penalty0 (1):\penalty0 57--85,
  March 2001.
\newblock ISSN 00043702.
\newblock \doi{10.1016/S0004-3702(01)00058-3}.
\newblock URL
  \url{http://linkinghub.elsevier.com/retrieve/pii/S0004370201000583}.

\bibitem[Jiang(2023)]{jiang_minimizing_2023}
Haotian Jiang.
\newblock Minimizing {Convex} {Functions} with {Rational} {Minimizers}.
\newblock \emph{Journal of the ACM}, 70\penalty0 (1):\penalty0 1--27, February
  2023.
\newblock ISSN 0004-5411, 1557-735X.
\newblock \doi{10.1145/3566050}.
\newblock URL \url{https://dl.acm.org/doi/10.1145/3566050}.

\bibitem[Kaminski et~al.(2017)Kaminski, Katoen, Matheja, and
  Olmedo]{kaminski_weakest_2017}
Benjamin~Lucien Kaminski, Joost-Pieter Katoen, Christoph Matheja, and Federico
  Olmedo.
\newblock Weakest {Precondition} {Reasoning} for {Expected} {Run}-{Times} of
  {Probabilistic} {Programs}, November 2017.
\newblock URL \url{http://arxiv.org/abs/1601.01001}.
\newblock arXiv:1601.01001 [cs.LO].

\bibitem[Knuth(1997)]{KnuthVol1}
Donald~E. Knuth.
\newblock \emph{The art of computer programming, volume 1 (3rd ed.):
  fundamental algorithms}.
\newblock Addison Wesley Longman Publishing Co., Inc., USA, 1997.
\newblock ISBN 0201896834.

\bibitem[Kratsch \& Neumann(2013)Kratsch and
  Neumann]{kratsch_fixed-parameter_2013}
Stefan Kratsch and Frank Neumann.
\newblock Fixed-{Parameter} {Evolutionary} {Algorithms} and the {Vertex}
  {Cover} {Problem}.
\newblock \emph{Algorithmica}, 65\penalty0 (4):\penalty0 754--771, April 2013.
\newblock ISSN 1432-0541.
\newblock \doi{10.1007/s00453-012-9660-4}.
\newblock URL \url{https://doi.org/10.1007/s00453-012-9660-4}.

\bibitem[Kwon et~al.(2023)Kwon, Li, Zhuang, Sheng, Zheng, Yu, Gonzalez, Zhang,
  and Stoica]{kwon_efficient_2023}
Woosuk Kwon, Zhuohan Li, Siyuan Zhuang, Ying Sheng, Lianmin Zheng, Cody~Hao Yu,
  Joseph~E. Gonzalez, Hao Zhang, and Ion Stoica.
\newblock Efficient memory management for large language model serving with
  {PagedAttention}.
\newblock In \emph{Proceedings of the 29th Symposium on Operating Systems
  Principles ({SOSP})}, pp.\  611--626, 2023.

\bibitem[Lange et~al.(2025)Lange, Imajuku, and Cetin]{lange_shinkaevolve_2025}
Robert~Tjarko Lange, Yuki Imajuku, and Edoardo Cetin.
\newblock {ShinkaEvolve}: {Towards} {Open}-{Ended} {And} {Sample}-{Efficient}
  {Program} {Evolution}, September 2025.
\newblock URL \url{http://arxiv.org/abs/2509.19349}.
\newblock arXiv:2509.19349 [cs].

\bibitem[Li et~al.(2024)Li, Wang, He, Wu, Xu, and Wen]{li_formalization_2024}
Chenyi Li, Ziyu Wang, Wanyi He, Yuxuan Wu, Shengyang Xu, and Zaiwen Wen.
\newblock Formalization of {Complexity} {Analysis} of the {First}-order
  {Algorithms} for {Convex} {Optimization}, July 2024.
\newblock URL \url{http://arxiv.org/abs/2403.11437}.
\newblock arXiv:2403.11437 [math].

\bibitem[Li et~al.(2026)Li, Yang, He, Zhao, Zhao, Tang, Yang, Gupta, Su, and
  Jin]{li_goedel-code-prover_2026}
Zenan Li, Ziran Yang, Deyuan He, Haoyu Zhao, Andrew Zhao, Shange Tang, Kaiyu
  Yang, Aarti Gupta, Zhendong Su, and Chi Jin.
\newblock Goedel-{Code}-{Prover}: {Hierarchical} {Proof} {Search} for {Open}
  {State}-of-the-{Art} {Code} {Verification}, March 2026.
\newblock URL \url{http://arxiv.org/abs/2603.19329}.
\newblock arXiv:2603.19329 [cs.SE].

\bibitem[Lin et~al.(2025)Lin, Tang, Lyu, Yang, Chung, Zhao, Jiang, Geng, Ge,
  Sun, Wu, Gesi, Lu, Acuna, Yang, Lin, Choi, Chen, Arora, and
  Jin]{lin_goedel-prover-v2_2025}
Yong Lin, Shange Tang, Bohan Lyu, Ziran Yang, Jui-Hui Chung, Haoyu Zhao, Lai
  Jiang, Yihan Geng, Jiawei Ge, Jingruo Sun, Jiayun Wu, Jiri Gesi, Ximing Lu,
  David Acuna, Kaiyu Yang, Hongzhou Lin, Yejin Choi, Danqi Chen, Sanjeev Arora,
  and Chi Jin.
\newblock Goedel-{Prover}-{V2}: {Scaling} {Formal} {Theorem} {Proving} with
  {Scaffolded} {Data} {Synthesis} and {Self}-{Correction}, August 2025.
\newblock URL \url{http://arxiv.org/abs/2508.03613}.
\newblock arXiv:2508.03613 [cs.CL].

\bibitem[Liu et~al.(2024)Liu, Tong, Yuan, Lin, Luo, Wang, Lu, and
  Zhang]{liu_evolution_2024}
Fei Liu, Xialiang Tong, Mingxuan Yuan, Xi~Lin, Fu~Luo, Zhenkun Wang, Zhichao
  Lu, and Qingfu Zhang.
\newblock Evolution of {Heuristics}: {Towards} {Efficient} {Automatic}
  {Algorithm} {Design} {Using} {Large} {Language} {Model}.
\newblock In \emph{Proceedings of the 41st {International} {Conference} on
  {Machine} {Learning}}, pp.\  32201--32223. PMLR, July 2024.
\newblock URL \url{https://proceedings.mlr.press/v235/liu24bs.html}.

\bibitem[Liu et~al.(2026)Liu, Yao, Guo, Yang, Lin, Zhao, Tong, Mao, Lu, Wang,
  Yuan, and Zhang]{liu_systematic_2026}
Fei Liu, Yiming Yao, Ping Guo, Zhiyuan Yang, Xi~Lin, Zhe Zhao, Xialiang Tong,
  Kun Mao, Zhichao Lu, Zhenkun Wang, Mingxuan Yuan, and Qingfu Zhang.
\newblock A {Systematic} {Survey} on {Large} {Language} {Models} for
  {Algorithm} {Design}.
\newblock \emph{ACM Computing Surveys}, 58\penalty0 (8):\penalty0 1--32, June
  2026.
\newblock ISSN 0360-0300, 1557-7341.
\newblock \doi{10.1145/3787585}.
\newblock URL \url{https://dl.acm.org/doi/10.1145/3787585}.

\bibitem[Luby et~al.(1993)Luby, Sinclair, and Zuckerman]{luby_optimal_1993}
Michael Luby, Alistair Sinclair, and David Zuckerman.
\newblock Optimal speedup of {Las} {Vegas} algorithms.
\newblock \emph{Information Processing Letters}, 47\penalty0 (4):\penalty0
  173--180, 1993.
\newblock ISSN 0020-0190.
\newblock \doi{https://doi.org/10.1016/0020-0190(93)90029-9}.
\newblock URL
  \url{https://www.sciencedirect.com/science/article/pii/0020019093900299}.

\bibitem[Motwani \& Raghavan(1995)Motwani and Raghavan]{Motwani_Raghavan_1995}
Rajeev Motwani and Prabhakar Raghavan.
\newblock \emph{Randomized Algorithms}.
\newblock Cambridge University Press, 1995.

\bibitem[Neumann \& Wegener(2007)Neumann and Wegener]{neumann_randomized_2007}
Frank Neumann and Ingo Wegener.
\newblock Randomized local search, evolutionary algorithms, and the minimum
  spanning tree problem.
\newblock \emph{Theoretical Computer Science}, 378\penalty0 (1):\penalty0
  32--40, June 2007.
\newblock ISSN 0304-3975.
\newblock \doi{10.1016/j.tcs.2006.11.002}.
\newblock URL
  \url{https://www.sciencedirect.com/science/article/pii/S0304397506008589}.

\bibitem[Novikov et~al.(2025)Novikov, Vũ, Eisenberger, Dupont, Huang, Wagner,
  Shirobokov, Kozlovskii, Ruiz, Mehrabian, Kumar, See, Chaudhuri, Holland,
  Davies, Nowozin, Kohli, and Balog]{novikov_alphaevolve_2025}
Alexander Novikov, Ngân Vũ, Marvin Eisenberger, Emilien Dupont, Po-Sen Huang,
  Adam~Zsolt Wagner, Sergey Shirobokov, Borislav Kozlovskii, Francisco J.~R.
  Ruiz, Abbas Mehrabian, M.~Pawan Kumar, Abigail See, Swarat Chaudhuri, George
  Holland, Alex Davies, Sebastian Nowozin, Pushmeet Kohli, and Matej Balog.
\newblock {AlphaEvolve}: {A} coding agent for scientific and algorithmic
  discovery, June 2025.
\newblock URL \url{http://arxiv.org/abs/2506.13131}.
\newblock arXiv:2506.13131 [cs].

\bibitem[Oliveto et~al.(2008)Oliveto, He, and Yao]{oliveto_analysis_2008}
Pietro~S. Oliveto, Jun He, and Xin Yao.
\newblock Analysis of population-based evolutionary algorithms for the vertex
  cover problem.
\newblock In \emph{2008 {IEEE} {Congress} on {Evolutionary} {Computation}
  ({IEEE} {World} {Congress} on {Computational} {Intelligence})}, pp.\
  1563--1570, June 2008.
\newblock \doi{10.1109/CEC.2008.4631000}.
\newblock URL \url{https://ieeexplore.ieee.org/abstract/document/4631000}.

\bibitem[Ospanov et~al.(2025)Ospanov, Farnia, and
  Yousefzadeh]{ospanov_apollo_2025}
Azim Ospanov, Farzan Farnia, and Roozbeh Yousefzadeh.
\newblock {APOLLO}: {Automated} {LLM} and {Lean} {Collaboration} for {Advanced}
  {Formal} {Reasoning}, May 2025.
\newblock URL \url{https://arxiv.org/abs/2505.05758v5}.

\bibitem[Robbins et~al.(2026)Robbins, Lawless, Udell, and
  Vitercik]{robbins_flare_2026}
Henry Robbins, Connor Lawless, Madeleine Udell, and Ellen Vitercik.
\newblock {FLARE}: {Verifying} {MILP} {Reformulations} with {LLM}-{Based}
  {Theorem} {Proving}, August 2026.
\newblock URL \url{http://arxiv.org/abs/2608.25220}.
\newblock arXiv:2608.25220 [cs.AI].

\bibitem[Romera-Paredes et~al.(2024)Romera-Paredes, Barekatain, Novikov, Balog,
  Kumar, Dupont, Ruiz, Ellenberg, Wang, Fawzi, Kohli, and
  Fawzi]{romeraparedes_mathematical_2024}
Bernardino Romera-Paredes, Mohammadamin Barekatain, Alexander Novikov, Matej
  Balog, M.~Pawan Kumar, Emilien Dupont, Francisco J.~R. Ruiz, Jordan~S.
  Ellenberg, Pengming Wang, Omar Fawzi, Pushmeet Kohli, and Alhussein Fawzi.
\newblock Mathematical discoveries from program search with large language
  models.
\newblock \emph{Nature}, 625\penalty0 (7995):\penalty0 468--475, January 2024.
\newblock \doi{10.1038/s41586-023-06924-6}.

\bibitem[Sch\"affeler \& Abdulaziz(2025)Sch\"affeler and
  Abdulaziz]{schaffeler_formally_2025}
Maximilian Sch\"affeler and Mohammad Abdulaziz.
\newblock Formally {Verified} {Approximate} {Policy} {Iteration}.
\newblock \emph{Proceedings of the AAAI Conference on Artificial Intelligence},
  39\penalty0 (25):\penalty0 26659--26667, April 2025.
\newblock ISSN 2374-3468.
\newblock \doi{10.1609/aaai.v39i25.34868}.
\newblock URL \url{https://ojs.aaai.org/index.php/AAAI/article/view/34868}.
\newblock Number: 25.

\bibitem[Sothanaphan(2026)]{sothanaphan_resolution_2026}
Nat Sothanaphan.
\newblock Resolution of {Erdős} {Problem} \#728: a writeup of {Aristotle}'s
  {Lean} proof, January 2026.
\newblock URL \url{http://arxiv.org/abs/2601.07421}.
\newblock arXiv:2601.07421 [math.NT].

\bibitem[Vajjha et~al.(2021)Vajjha, Shinnar, Trager, Pestun, and
  Fulton]{vajjha_certrl_2021}
Koundinya Vajjha, Avraham Shinnar, Barry Trager, Vasily Pestun, and Nathan
  Fulton.
\newblock {CertRL}: formalizing convergence proofs for value and policy
  iteration in {Coq}.
\newblock In \emph{Proceedings of the 10th {ACM} {SIGPLAN} {International}
  {Conference} on {Certified} {Programs} and {Proofs}}, {CPP} 2021, pp.\
  18--31, New York, NY, USA, January 2021. Association for Computing Machinery.
\newblock ISBN 978-1-4503-8299-1.
\newblock \doi{10.1145/3437992.3439927}.
\newblock URL \url{https://dl.acm.org/doi/10.1145/3437992.3439927}.

\bibitem[{van Stein} \& Bäck(2025){van Stein} and Bäck]{stein_llamea_2025}
Niki {van Stein} and Thomas Bäck.
\newblock {LLaMEA}: {A} {Large} {Language} {Model} {Evolutionary} {Algorithm}
  for {Automatically} {Generating} {Metaheuristics}.
\newblock \emph{IEEE Transactions on Evolutionary Computation}, 29\penalty0
  (2):\penalty0 331--345, April 2025.
\newblock ISSN 1941-0026.
\newblock \doi{10.1109/TEVC.2024.3497793}.
\newblock URL \url{https://ieeexplore.ieee.org/abstract/document/10752628}.

\bibitem[Varambally et~al.(2025)Varambally, Voice, Sun, Chen, Yu, and
  Ye]{varambally_hilbert_2025}
Sumanth Varambally, Thomas Voice, Yanchao Sun, Zhifeng Chen, Rose Yu, and
  Ke~Ye.
\newblock Hilbert: {Recursively} {Building} {Formal} {Proofs} with {Informal}
  {Reasoning}, September 2025.
\newblock URL \url{https://arxiv.org/abs/2509.22819v2}.

\bibitem[Wegener(2002)]{wegener_methods_2002}
Ingo Wegener.
\newblock Methods for the {Analysis} of {Evolutionary} {Algorithms} on
  {Pseudo}-{Boolean} {Functions}.
\newblock In Ruhul Sarker, Masoud Mohammadian, and Xin Yao (eds.),
  \emph{Evolutionary {Optimization}}, pp.\  349--369. Springer US, Boston, MA,
  2002.
\newblock ISBN 978-0-306-48041-6.
\newblock \doi{10.1007/0-306-48041-7_14}.
\newblock URL \url{https://doi.org/10.1007/0-306-48041-7_14}.

\end{thebibliography}

\newpage
\appendix
\raggedbottom

\section{Formal Statement of the Pipeline Stages}\label{app:pipeline-stages}

\Cref{tab:pipeline-stages} states every stage of the pipeline formally;
the stages are described informally in \Cref{sec:pipeline}.

\begin{footnotesize}
\begin{longtable}{@{}p{0.9cm}p{3.0cm}p{9.0cm}@{}}
  \caption{The stages of the pipeline that
  generates one algorithm $A^{i+1}=(\Y,\init,\step,\query)$ and proves it
  faster than its predecessor, restated in the $(\X,\Y,\Z,\R)$ formalism of \Cref{algo:blackbox}. \texttt{[I]}
    is an informal stage producing prose (no Lean code) with candidate selection (see \Cref{sec:pipeline-algorithm} and \Cref{sec:pipeline-potential}), \texttt{[F*]} is a formal stage that introduces new Lean definitions used by later stages,
    and \texttt{[F-]} a formal stage that only proves properties of definitions introduced earlier;
    in both, the Lean declarations produced are checked by Lean and carried forward. $f\in\F$ denotes a fixed fitness function
    (an element of the problem class, e.g.\ $\onemax$ composed with an unknown target
    bitstring), so that $\reachableWith(f,y)$ (as opposed to $\reachable(y)$) restricts to states
    reachable when the \emph{true} fitness fed to $\step$ is always $f$. $z,z_1,z_2\in\Z$ denote
    random draws; $\unif(z)\in\X$ is the bitstring obtained by sampling each coordinate of $z$
    against threshold $1/2$. $n$ is the bitstring length, and a `prev' subscript refers to the
    previous algorithm/potential pair in the outer self-improvement loop.}
  \label{tab:pipeline-stages}\\
  \toprule
  Stage & Name & Formalisation\\
  \midrule
  \endfirsthead
  \multicolumn{3}{l}{\emph{(Table \thetable{} continued)}}\\
  \toprule
  Stage & Name & Formalisation\\
  \midrule
  \endhead
  \bottomrule
  \endfoot
  \bottomrule
  \endlastfoot

  \multicolumn{3}{@{}l}{\textbf{Algorithm construction (Stages 0--8)}}\\[2pt]

  \texttt{0} & \texttt{DesignLock} \newline\texttt{[I]}
    & Commit to $\Y,\init,\step,\query$ and a qualitative account of which component of
      $y\in\Y$ tracks $f$; no numeric constants yet.\\[4pt]

  \texttt{1} & \texttt{Construct-}\newline\texttt{AlgorithmCom-}\newline\texttt{ponents [F*]}
    & Define $\Y$, $\init:\Z\to\Y$, $\step:\Z\times\Y\times\R\to\Y$, $\query:\Y\to\X$.\\[4pt]

  \texttt{1c} & \texttt{StepMeasur-}\newline\texttt{ability [F-]}
    & $\step$ is measurable.\\[4pt]

  \texttt{2} & \texttt{Completeness1}\newline\texttt{[F-]}
    & $\forall x\in\X,\forall y\in\Y,\ \reachable(y)\rightarrow
       \exists z\in\Z,\exists f\in\R,\ \query(\step(z,y,f))=x$.\\[4pt]

  \texttt{3} & \texttt{NonAbsorbing}\newline\texttt{[F-]}
    & $\forall y\in\Y,\ \reachable(y)\rightarrow\forall f\in\R,\
       \exists z_1,z_2\in\Z,\ \query(\step(z_1,y,f))\neq\query(\step(z_2,y,f))$.\\[4pt]

  \texttt{4} & \texttt{NotRandom-}\newline\texttt{Search [F-]}
    & $\exists y\in\Y,\ \reachable(y)\wedge\exists z\in\Z,\exists f\in\R,\
       \query(\step(z,y,f))\neq\unif(z)$.\\[4pt]

  \texttt{5} & \texttt{Adaptiveness}\newline\texttt{[F-]}
    & $\exists y\in\Y,\ \reachable(y)\wedge\exists z,z'\in\Z,\exists f_1,f_2,f\in\R,$\newline
      $\query(\step(z',\step(z,y,f_1),f))\neq\query(\step(z',\step(z,y,f_2),f))$\newline
      (a two-step strengthening of the usual one-step adaptiveness property: the fitness fed at
      the \emph{first} step must still move the query one step later).\\[4pt]

  \texttt{6} & \texttt{QueryNon-}\newline\texttt{Injectivity [F-]}
    & $\exists y_1,y_2\in\Y,\ y_1\neq y_2\wedge\query(y_1)=\query(y_2)$.\\[4pt]

  \texttt{7} & \texttt{AlgorithmInfra-}\newline\texttt{structure [F*]}
    & Equip $\Y$ with the five instances (measurable space, topology, pseudo-metrizability,
      Borel $\sigma$-algebra, second-countability) the framework requires.\\[4pt]

  \texttt{7a} & \texttt{InitMeasur-}\newline\texttt{ability [F-]}
    & $\init$ is measurable.\\[4pt]

  \texttt{7b} & \texttt{QueryMeasur-}\newline\texttt{ability [F-]}
    & $\query$ is measurable.\\[4pt]

  \texttt{8} & \texttt{Construct-}\newline\texttt{AlgorithmStruc-}\newline\texttt{ture [F*]}
    & Package $A=(\Y,\init,\step,\query)$, together with the Stage 1c/7/7a/7b witnesses, into a
      single algorithm record.\\[4pt]

  \multicolumn{3}{@{}l}{\textbf{Potential design (Stages 9a--9f)}}\\[2pt]

  \texttt{9a} & \texttt{Potential-}\newline\texttt{Design [I]}
    & Rank candidate potentials $\Phi:\F\times\Y\to[0,1]$ by whether $\step$ can be shown, even
      informally, to decrease $\Phi$ faster than $\unif$.\\[4pt]

  \texttt{9b} & \texttt{Define-}\newline\texttt{Potential [F*]}
    & Define $\Phi:\F\times\Y\to\R$ with $\forall f\in\F,\forall y\in\Y,\ 0\le \Phi(f,y)\le 1$.\\[4pt]

  \texttt{9c} & \texttt{PotentialOpt}\newline\texttt{[F-]}
    & $\forall f\in\F,\forall y\in\Y,\ \Phi(f,y)\le 0 \iff \forall x\in\X,\ f(x)\le f(\query(y))$
      --- $\Phi(f,\cdot)$ vanishes exactly at states whose query is $f$-optimal.\\[4pt]

  \texttt{9d} & \texttt{PotentialFine-}\newline\texttt{Grained [F-]}
    & $\exists f\in\F,\exists y\in\Y,\ 0<\Phi(f,y)<1$ --- rules out a $\{0,1\}$-valued (indicator)
      potential.\\[4pt]

  \texttt{9e} & \texttt{PotentialMeas-}\newline\texttt{urability [F-]}
    & $\forall f\in\F,\ \Phi(f,\cdot)$ is measurable.\\[4pt]

  \texttt{9f} & \texttt{ConstructPoten-}\newline\texttt{tialStructure}\newline\texttt{[F*]}
    & Package $\Phi$, together with the Stage 9b--9e witnesses, into a single potential
      record.\\[4pt]

  \multicolumn{3}{@{}l}{\textbf{Drift design and runtime bound (Stages 10a--10j)}}\\[2pt]

  \texttt{10a} & \texttt{DriftDesign}\newline\texttt{[I]}
    & Work out closed forms, against the now-fixed $A$ and $\Phi$, for $\varepsilon,p,p_{\text{jump}}:\N\to
      \R_{>0}$ and for the four draw sets $\Good,\Worsen,\BigWorsen,\Jump:\F\times\Y\to
      2^{\Z}$ below.\\[4pt]

  \texttt{10b} & \texttt{DefineDrift-}\newline\texttt{Constants [F*]}
    & Define $\varepsilon(n),p(n),p_{\text{jump}}(n)$ as closed-form expressions in $n$.\\[4pt]

  \texttt{10c} & \texttt{DriftConstant-}\newline\texttt{Properties [F-]}
    & $\varepsilon(n)>0,\ p(n)>0,\ p_{\text{jump}}(n)\in(0,1)$, and (the speed-up certificate)
      $\exists\delta>0,\exists n_0,\forall n>n_0,$\newline
      $(1+\delta)\,\varepsilon_{\mathrm{prev}}(n)\,p_{\mathrm{prev}}(n) <
      \varepsilon(n)\,p(n)$.\\[4pt]

  \texttt{10d} & \texttt{DefineDrift-}\newline\texttt{Sets [F*]}
    & Define $\Good,\Worsen,\BigWorsen,\Jump:\F\times\Y\to 2^{\Z}$ as explicit sets of
      draws.\\[4pt]

  \texttt{10e} & \texttt{DriftSetMeas-}\newline\texttt{urability [F-]}
    & For each $S\in\{\Good,\Worsen,\BigWorsen,\Jump\}$: $\{(y,z)\in\Y\times\Z : z\in
      S(f,y)\}$ is measurable.\\[4pt]

  \texttt{10f} & \texttt{PositiveDrift}\newline\texttt{[F-]}
    & $\forall y\in\Y,\ \reachableWith(f,y)\wedge \Phi(f,y)>0 \rightarrow \forall z\in\Good(f,y),$
      \newline
      $\Phi\bigl(f,\step(z,y,f(\query(y)))\bigr) \le \Phi(f,y)-\varepsilon(n)$.\\[4pt]

  \texttt{10g} & \texttt{NegativeDrift-}\newline\texttt{AndJump [F-]}
    & For $y$ with $\reachableWith(f,y)$, $\Phi(f,y)>0$, and every $z\in\Z$:\newline
      $\Phi(f,\step(z,y,f(\query(y))))>\Phi(f,y) \rightarrow z\in\Worsen(f,y)$,\newline
      $\Phi(f,\step(z,y,f(\query(y))))>\Phi(f,y)+\varepsilon(n) \rightarrow z\in\BigWorsen(f,y)$;
      \newline
      and, for every $y$ with $\reachableWith(f,y)$: $\forall z\in\Jump(f,y),\
      \Phi(f,\step(z,y,f(\query(y))))=0$.\\[4pt]

  \texttt{10h} & \texttt{DriftSetMass}\newline\texttt{[F-]}
    & For $\reachableWith(f,y)$ with $\Phi(f,y)>0$:
      $p(n)\le\prob{z\in\Good(f,y)}$,\newline
      $\prob{z\in\Worsen(f,y)}\le p(n)/2$,\newline
      $\prob{z\in\BigWorsen(f,y)}\le\varepsilon(n)\,p(n)/4$.\\[4pt]

  \texttt{10i} & \texttt{JumpMass}\newline\texttt{[F-]}
    & $\forall y\in\Y,\ p_{\text{jump}}(n)\le\prob{z\in\Jump(f,y)}$ --- holds at \emph{every} $y$, unlike
      Stage 10h, since it underwrites integrability of the hitting time.\\[4pt]

  \texttt{10j} & \texttt{RuntimeBound}\newline\texttt{[F-]}
    & $\expect{T_A}\ \le\ \dfrac{4}{\varepsilon(n)\,p(n)}$.\\

\end{longtable}
\end{footnotesize}

\section{Proofs of the drift condition}\label{app:drift-proofs}

The additive drift theorem used in \Cref{sec:runtime-analysis} is the following; a proof is given by
\citet{corus_level-based_2018}.

\begin{theorem}[Additive drift theorem]\label{lemma:pol-drift}
Let $(Z_t)_{t\in\mathbb{N}}$ be a discrete-time stochastic process in
$[0,\infty)$ adapted to any filtration $(\filt{t})_{t\in\mathbb{N}}$.
For any $a\geq 0$, define $T_a := \min\{t\in\mathbb{N} \mid Z_t \leq a\}$.
If for some $\varepsilon>0$ and all $t\in\mathbb{N}$: (1)~$\expect{Z_{t+1} - Z_{t} + \varepsilon \;; t<T_{a} \mid
\filt{t}} \leq 0$, (2)~$\expect{Z_t}<\infty$, and (3)~$\expect{T_a} < \infty$, then $\expect{T_a \mid \filt{0}} \leq Z_0 / \varepsilon$.
\end{theorem}

We present the proof of Corollary~\ref{corollary:simple-drift}. The
Lean formalisation of this proof, which was developed as part of the
staged pipeline, follows almost exactly this proof.

\begin{proof}[Proof of Corollary~\ref{corollary:simple-drift}]
  Define
  \begin{align*}
    E_1 & := X_{t+1} \leq X_t -\varepsilon \\
    E_2 & := X_t - \varepsilon < X_{t+1} \leq X_t \\
    E_3 & := X_t < X_{t+1} \leq X_t + \varepsilon \\
    E_4 & := X_t + \varepsilon < X_{t+1}
  \end{align*}
  Note that these definitions immediately imply 
  \begin{align}
    \indf{t<T}\indfe{E_1}(X_t-X_{t+1}) & \geq \indf{t<T}\indfe{E_1}\varepsilon \label{eq:simpdrift1}\\
    \indf{t<T}\indfe{E_2}(X_t-X_{t+1}) & \geq 0 \label{eq:simpdrift2}\\
    \indf{t<T}\indfe{E_3}(X_t-X_{t+1}) & \geq -\indf{t<T}\indfe{E_3}\varepsilon \label{eq:simpdrift3}\\
    \indf{t<T}\indfe{E_4}(X_t-X_{t+1}) & \geq -\indf{t<T}\indfe{E_4}\label{eq:simpdrift4}
  \end{align}
  where the last inequality follows because $X_t\in[0,1]$ for all $t$.  
  We now have
  \begin{align*}
    \indf{t<T}(\varepsilon p/4)
    & = \indf{t<T}(\varepsilon p - \varepsilon p / 2 - \varepsilon p / 4)\\
    \intertext{by assumptions 1-3}
    & \leq \expect{\indf{t<T}\indfe{E_1}\varepsilon\mid \filt{t}}
         - \expect{\indf{t<T}\indfe{E_3\cup E_4}\varepsilon\mid \filt{t}}
         - \expect{\indf{t<T}\indfe{E_4}\mid \filt{t}}      \\
    \intertext{since $\indfe{E_3\cup E_4}\geq \indfe{E_3}$}
    & \leq \expect{\indf{t<T}\indfe{E_1}\varepsilon\mid \filt{t}}
         - \expect{\indf{t<T}\indfe{E_3}\varepsilon\mid \filt{t}}
         - \expect{\indf{t<T}\indfe{E_4}\mid \filt{t}}      \\
    \intertext{by Eqs. (\ref{eq:simpdrift1}), (\ref{eq:simpdrift3}), and (\ref{eq:simpdrift4})}
    & \leq \expect{\indf{t<T}(\indfe{E_1} + \indfe{E_3} + \indfe{E_4})(X_t-X_{t+1})\mid \filt{t}}\\
    \intertext{by Eq. (\ref{eq:simpdrift2}) and $\indfe{E_1}+\indfe{E_2}+\indfe{E_3}+\indfe{E_4}=1$}
    & \leq \expect{\indf{t<T}(X_t-X_{t+1})\mid \filt{t}}.
  \end{align*}
  The result now follows by rearranging the inequality above.
\end{proof}

\section{Implementation and prompt}\label{sec:implementation}

\paragraph{Implementation.}
The pipeline Python code was initially based on the Goedel-Prover-V2
codebase \citep{lin_goedel-prover-v2_2025}.  Every candidate
declaration is checked by Lean~4 on a Kimina Lean server
\citep{dos_santos_kimina_2025}, and the Lean errors are returned to
the model in the next attempt. The number of attempts allowed before
backtracking to a previous stage is determined by a Luby sequence
\citep{luby_optimal_1993} $L=(1,1,2,1,1,2,4,\ldots)$.  In particular, each stage has its own
copy of the sequence with unit 3, so it is allowed
$3,3,6,3,3,6,12,\ldots$ attempts on successive visits. When a stage
exhausts its allowance, the pipeline backtracks by one stage on the
first failure of that stage, and by more stages one each further
failure. In addition, the whole pipelinen is restarted from Stage~0
when $200,L_k$ attempts have been used in the $k$-th run of the
pipeline, or earlier when 100 consecutive attempts make no progress
beyond the furthest stage reached. All runs in \Cref{tab:runs} use
these parameters.

The responses to the prompts of each stage come from a
\emph{backend}. With the \emph{CLI backend}, the pipeline writes each
prompt to standard output and reads the response from standard input,
so that the pipeline can be driven by an external agent; rows 1--18 of
\Cref{tab:runs} use this backend with Claude models. With the
\emph{vLLM backend}, the pipeline serves open-weight models itself
with vLLM \citep{kwon_efficient_2023}, sampling many candidates per
attempt; rows 19--21 use this backend with Goedel-Prover-V2-32B
\citep{lin_goedel-prover-v2_2025} for the formal stages and Gemma~4
26B-A4B \citep{gemma_team_gemma_2026} for the prose stages.  The
stages, their Lean statements, and the checks are the same for both
backends.

\paragraph{The prompt.} Every stage prompt starts with the same global
block, followed by the accepted declarations so far and the task of
the stage. The global block describes the framework: the black-box
algorithm structure of \Cref{algo:blackbox}, the problem-class
interface and the baseline algorithm, whose
definitions the agent can use as a template.  It then lists an API of
library results, whose signatures are checked against the Lean build.
These include samplers that turn the random source into bitstrings,
with a fixed or a state-dependent per-bit bias, together with inverse
lemmas that construct a random input producing any given bitstring. They
also include lemmas about bitstrings, how to access elements of a bitstrings, etc.
Furthermore, a measurability toolkit provides two tactics
that decompose tuple-, \texttt{if}- and arithmetic-structured step and
potential functions into leaves, and gives the idioms that close each
kind of leaf. For the drift stages, the block describes the random
source, the reachability predicates and an assembly theorem of
Stage~10j, and it explains which properties of that theorem constrain
the design, e.g.\ that a single $\varepsilon$ and $p$ must serve every
reachable non-optimal state. The problem card
(\Cref{app:benchmarks,app:runs}) is spliced in as the only
problem-specific part. Finally, a proof cookbook collects recurring
Lean pitfalls and their fixes, such as 
\texttt{if b then} is preferred over \texttt{if}~$\neg$\,\texttt{b},
never to call \texttt{decide} on goals involving a symbolic $n$, how to show that
two state tuples differ, and to prove measurability with the toolkit
rather than with \texttt{aesop}, whose fixed search budget times out
on nested \texttt{if}s. After a failed attempt the agent sees the
failing declaration, the Lean error, its earlier attempts at the stage
and a list of facts verified in earlier stages.

\section{Theoretical benchmark functions}\label{app:benchmarks}

For $x\in\{0,1\}^n$ write $|x|_1=\sum_{i=1}^n x_i$ for the number of
ones. The three benchmark functions are
\begin{align*}
  \onemax(x) &= |x|_1 = \sum_{i=1}^n x_i,\\
  \leadingones(x) &= \sum_{i=1}^{n}\prod_{j=1}^{i} x_j,\\
  \jump_k(x) &= \begin{cases}
    k+|x|_1 & \text{if } |x|_1\le n-k \text{ or } |x|_1=n,\\
    n-|x|_1 & \text{otherwise,}
  \end{cases}
\end{align*}
so \onemax counts the ones, \leadingones counts the ones before the first zero, and \jump$_k$
behaves like \onemax except for a gap of width $k$ below the optimum $1^n$, in which the fitness
points away from the optimum. For the $(1{+}1)$ EA with mutation rate $1/n$ the expected
optimisation times are $\Theta(n\log n)$ on \onemax, $\Theta(n^2)$ on \leadingones and $\Theta(n^k)$
on \jump$_k$ for constant $k\ge2$ \citep{droste_analysis_2002}.

A black-box algorithm must not depend on the instance, so each function is turned into a class of
instances by hiding its optimum. Write $\oplus$ for the bitwise exclusive or and $\bar z$ for the
complement of $z$. Each permutation $\sigma$ of $[n]$ induces a permutation of
$\{0,1\}^n$, which we also denote by $\sigma$, defined by
$\sigma(x)_i \coloneqq x_{\sigma(i)}$ for all $i\in[n]$. The problem
cards used in our experiments define
\begin{align*}
  F_z(x) &= \onemax(x\oplus z), && z\in\{0,1\}^n,\\
  F_{z,\sigma}(x) &= \leadingones\bigl(\sigma(x\oplus z)\bigr), && z\in\{0,1\}^n,\ \sigma \text{ a permutation},\\
  F_z(x) &= \jump_k(x\oplus z), && z\in\{0,1\}^n.
\end{align*}
Each instance has the unique optimum $\bar z$. For \jump$_k$ the gap width $k\ge2$ is a constant of the
class that the agent does not know, so every proof must be uniform in $k$. Algorithms whose
mutation treats all bits alike, such as the $(1{+}1)$ EA, have the same expected optimisation time on
every instance of a class as on the underlying function.

\section{\cycleuncover: definition and runtime of the (1+1) EA}\label{app:cycle-lb}

\begin{algorithm}[ht]
    \caption{(1+1) EA}
    \begin{algorithmic}[1]
        \State $t := 0$
        \State Choose an initial bit string $x_0\in \{0,1\}^n$ uniformly at random
        \While{forever}
        \State Create $x'$ by flipping each bit in $x_t$ independently
        with probability $p$ (mutation)
        \State $x_{t+1}:= x'$ if $f(x')\geq f(x)$, and $x_{t+1}:= x_t$
        otherwise (selection)
        \State $t:=t+1$
        \EndWhile
    \end{algorithmic}
    \label{algo:oneplusone}
\end{algorithm}

This section proves \Cref{thm:oneplusone-cycleuncover}. The objective $\eval$, the count
$\countones$ and the feasibility condition $\guard$ are those of \Cref{fig:cycleuncover}: $\guard(x)$
holds exactly when the chosen vertices form an independent set of the hidden cycle, and
$\eval(x)=2h+\countones(x)$ if $\guard(x)$ and $\eval(x)=\countones(x)$ otherwise.

\begin{proof}[Proof of \Cref{thm:oneplusone-cycleuncover}]
  Without loss of generality, we assume that $\sigma$ is the
  identity-permutation.

  We condition on the event that both endpoints $u$ and $v$ of the
  first edge are 1-bits in the initial search point, and that it
  contains at least $9n/20$ 1-bits.  By a union bound and a Chernoff bound \citep{Motwani_Raghavan_1995},
  the probability of this event is at least
  $1/4+(1-e^{-\Omega(n)})-1\geq 1/5$. If this event does not happen,
  we lower bound the runtime by 0.
    
  We consider two phases, the first phase consisting of the first
  $t=3en$ iterations, and the second phase consisting of the following
  $e^n$ iterations. We consider phase 1 a failure (from the
  perspective of the analysis) if bit $u$ or $v$ is flipped during
  this phase. The probability of this failure event is upper bounded
  by $1-(1-1/n)^{2t}\leq 1-e^{-2t/(n-1)}$. If phase 1 is not a
  failure, then $\guard(x)$ is never satisfied, and the fitness is
  $\countones(x)$, i.e., the number of 1-bits is non-decreasing.
    
  We call an iteration in the first phase successful if either the
  current search point has at least $5n/6$ 1-bits, or if mutation
  increases the number of 1-bits. The probability of a successful
  iteration is at least
  \begin{align*}
    p \geq \frac{n}{6}\left(1-\frac{1}{n}\right)^{n-1}\frac{1}{n}\geq \frac{1}{6e}.
  \end{align*}
  The number of successful iterations in phase 1 dominates
  stochastically a binomially distributed random variable with
  parameters $t$ and $p$. By a Chernoff bound, the probability that
  there are less than
  \begin{align*}
    tp(1-1/6) \geq (n/2)(1-1/6) = 5n/12
  \end{align*}
  successes is $e^{-\Omega(n)}$, and otherwise we call phase 1 a failure.

  Assuming no failure in phase 1, the current search point $x$ at the
  start of phase 2 contains $|x|\geq \min(5n/6,9n/20+5n/12)=5n/6$
  1-bits.  We now prove inductively that in each of the following
  $e^{n}$ iterations, the number of 1-bits in the current search point
  is at least $5n/6$. If the current point is infeasible with at least
  $5n/6$ 1-bits, an offspring is accepted only if it is feasible or
  has at least as many 1-bits. In order to satisfy $\guard(x)$ i.e.,
  the solution is feasible, it is necessary to have at most
  $h+1\leq n/2+1$ 1-bits (at most $h$ on the cycle, plus bit
  $\sigma(n)$ if $n$ is odd). Therefore, to satisfy $\guard(x)$, at
  least $n/3-1$ bits must
  be flipped simultaneously. By a union bound, this type of failure
  has probability at most $e^nn^{-\Omega(n)}=n^{-\Omega(n)}$.

  Conditioned on the good starting event and no failures, the running
  time of the algorithm is at least $e^n$. By a union bound, the
  probability of a failure in phase 1 or 2 is at most
  $1-e^{-2t/(n-1)}+e^{-\Omega(n)}+n^{-\Omega(n)}<c$ for some constant
  $c\in(0,1)$. The unconditional expected runtime is therefore at
  least $(1/5)(1-c)e^n=e^{\Omega(n)}$.
\end{proof}

\section{Algorithms, potentials and runtime bounds of all runs}\label{app:runs}

This appendix collects, for every run summarised in \Cref{tab:runs}, the synthesised algorithm, the
potential function and the runtime bound. The pseudocode is a human-readable transcription of the
frozen Lean definitions of \init\ and \step; in every run \query\ returns the first component of the
state. We write $\mathrm{bits}_p(u)\in\{0,1\}^n$ for the per-bit sampler that sets bit $i$ iff
$u_i\le p_i$ (so $\Pr[\mathrm{bits}_p(u)_i=1]=p_i$; a scalar $p$ means $p_i=p$ for all $i$),
$u^{+i}\coloneqq(u_i,u_{i+1},\dots)$ for the coordinates of the same draw from index $i$ on,
$\oplus$ for bitwise XOR, $|x|_1$ for the number of ones, $e_v$ for the $v$-th unit vector,
$\bar x$ for the complement, $[P]\in\{0,1\}$ for the indicator of $P$, and
$\mathrm{clamp}_{[a,b]}(x)\coloneqq\min\{b,\max\{a,x\}\}$ for clamping $x$ to $[a,b]$ (coordinatewise for vectors). Unless stated otherwise,
$f$ is the fitness of the pending query $q$ passed to \step, and a potential $\Phi_z(s)$ is $0$ if and only if
the query of $s$ is fitness-maximal. For staged runs the runtime theorem is
$\expect{T_A}\le 4/(\varepsilon_1 p_1)$ for all $n\ge 3$ and every instance $z$, with the drift
constants $\varepsilon_1$ and $p_1$ (the Good-set mass) given; ``machine-checked'' means the file
compiles with no \texttt{sorry}.

\paragraph{Problem cards.} \onemax: $F_z(x)=\onemax(x\oplus z)$, unique optimum $\bar z$.
\leadingones (card \texttt{lo}): an instance is a pair $(z,\sigma)$ of a string $z$ and a
permutation $\sigma$ of the positions, and $F_{z,\sigma}(x)=\leadingones\bigl(\sigma(x\oplus z)\bigr)$
with $\sigma(x)$ and \leadingones as in \Cref{app:benchmarks}; the
unique optimum is $\bar z$. The second version of the card also supplies the inequality
$e^{-m/(r-1)}\le(1-1/r)^m$ for $r>1$.
\jump$_k$ (card \texttt{jump\_k}): $F_z(x)=\jump_k(x\oplus z)$ with \jump$_k$ as in
\Cref{app:benchmarks}; the unique optimum is $\bar z$. The gap width $k\ge2$ is an opaque constant, so every proof must be uniform in $k$.
\submodmin, integer-valued (card \texttt{submodular\_minimisation\_integer\_valued\_additional\_advice}):
an instance is a submodular function $f:\{0,1\}^n\to\{0,1,\dots,B\}$ with $B=n^c$ for an opaque
class constant $c\ge1$ and $\min_x f(x)=0$, and $F_f(x)=-f(x)$. Besides the definitions, the card
supplies only facts about the integer range (every non-optimal point has $F_f\le-1$, and distinct
values differ by at least $1$) and the advice quoted in \Cref{app:run-submod}; it supplies no
lemmas about submodular functions.
\cycleuncover (card \texttt{laminate}): the objective $\eval$ of \Cref{fig:cycleuncover}, with
$h=\lfloor n/2\rfloor$, optimum value $3h$. \vertexcover, \emph{flat fitness} (cards
\texttt{vertex\_cover}, early \texttt{bandwidth\_vertex\_cover}): $F_z(x)=-|x|_1$ if $x$ is a
cover of $G_z$ and $|x|_1-2n$ otherwise. \vertexcover, \emph{graded fitness} (later
\texttt{bandwidth\_vertex\_cover}, \texttt{wcvc}, \texttt{well\_covered\_vertex\_cover}):
$F_z(x)=-\bigl((n+1)\,\mathrm{unc}_z(x)+|x|_1\bigr)$, where $\mathrm{unc}_z(x)$ is the number of
edges with no chosen endpoint; covers are exactly the points with $F_z\ge-n$. For the bandwidth
cards every edge $uv$ satisfies $|u-v|\le k$, and $k$ is known to the algorithm; the well-covered
cards restrict to graphs whose minimal covers are all minimum. The \texttt{wcvc} card is an anonymised
copy of the well-covered card with all names and derived facts removed.

\subsection{\onemax, open-weight models (staged; incomplete)}\label{app:run-open}

These runs used the vLLM backend, which serves the models itself, instead of the CLI backend used for the other runs (\Cref{sec:implementation}). The formal stages were
generated by Goedel-Prover-V2-32B with 256 samples per attempt at temperature 1, and the informal
stages (0, 9a, 10a) by Gemma~4 26B-A4B, both served with vLLM on 3 NVIDIA L40 GPUs (Sulis) or
4 NVIDIA GH200 GPUs (Isambard-AI). All runs are on \onemax with earlier versions of the staged
pipeline that have the same 28 stages as \Cref{tab:pipeline-stages}. The token usage of the formal
model is not logged by this backend, and the Lean code of these runs is not included.

\Cref{tab:open-runs} lists the 14 runs made with the three most recent of these pipeline versions
(September 2026). No run completed: each ended after 21--46 hours of wall-clock time. Ten runs got
stuck at Stage~4 (\texttt{NotRandomSearch}), three at Stage~1c or~2, and one reached Stage~9d
(\texttt{PotentialFineGrained}). Rows 19--21 of \Cref{tab:runs} are, in order, the most recent
run on each cluster with the most recent pipeline used there, and the run that got furthest.

\begin{table}[tb]
\centering
\caption{Open-weight runs on \onemax with the most recent pipeline versions. ``Last passed'' is the
last stage that was machine-checked.}
\label{tab:open-runs}
\scriptsize
\begin{tabular}{@{}llllrl@{}}
\toprule
Run & Cluster & Started & Pipeline version & Hours & Last passed\\
\midrule
\texttt{generous-dragon-779} & Sulis & 2026-09-21 & \texttt{optimisation} & 23.8 & 3\\
\texttt{driven-blaze-773} & Sulis & 2026-09-18 & \texttt{optimisation} & 23.8 & 1\\
\texttt{denim-paper-771} & Sulis & 2026-09-17 & \texttt{split\_v9} & 21.6 & 3\\
\texttt{super-moon-761} & Sulis & 2026-09-15 & \texttt{split\_v9} & 23.3 & 3\\
\texttt{bumbling-plant-742} & Sulis & 2026-09-13 & \texttt{split\_v9} & 23.2 & 3\\
\texttt{avid-waterfall-734} & Isambard-AI & 2026-09-12 & \texttt{split\_v9} & 23.6 & 3\\
\texttt{mild-dream-724} & Sulis & 2026-09-08 & \texttt{split\_v9} & 23.6 & 3\\
\texttt{absurd-puddle-723} & Isambard-AI & 2026-09-08 & \texttt{split\_v9} & 23.6 & 3\\
\texttt{peach-armadillo-718} & Isambard-AI & 2026-09-06 & \texttt{split\_v9} & 23.7 & 3\\
\texttt{serene-voice-715} & Sulis & 2026-09-06 & \texttt{split\_v9} & 23.8 & 1c\\
\texttt{visionary-snow-713} & Sulis & 2026-09-05 & \texttt{split\_v8} & 20.7 & 3\\
\texttt{pleasant-elevator-712} & Isambard-AI & 2026-09-04 & \texttt{split\_v8} & 23.7 & 1c\\
\texttt{golden-yogurt-711} & Isambard-AI & 2026-09-02 & \texttt{split\_v8} & 36.8 & 9c\\
\texttt{royal-snowflake-710} & Sulis & 2026-09-01 & \texttt{split\_v8} & 46.2 & 3\\
\bottomrule
\end{tabular}
\end{table}

All samples of the runs in rows 19--21 are logged, so the output tokens of Goedel-Prover-V2-32B
can be counted with its own tokeniser over the logged responses; the informal stages add at most
$7\times10^4$ tokens per run. Lean lines are counted as in \Cref{tab:runs}, at the last passed stage.

\emph{Row 19}. The frozen design samples
each bit with its own bias, initialised to $1/2$; after an improvement, every bias moves by $0.1$
towards the corresponding bit of the query, clamped to the interval $[0.1,0.9]$. This is a PBIL-like
estimation-of-distribution algorithm \citep{pbil1994}, but with too large margins. Stages 0--3 passed. Stage~4 is true for this design (the biases leave $1/2$ after
one improving step), but all $13\times256$ samples for it failed: 38\% contained no extractable
Lean code, 35\% left unsolved goals, 10\% did not parse and 6\% used unknown identifiers. The
pipeline backtracked to Stages 3, 2 and 1c, but never to the design stages. 5{,}896 samples,
25.1M output tokens, 23.8\,h.

\emph{Row 20}. Both designs, the
first and the one proposed after a backtrack to Stage~0, are PBIL-like: per-bit biases, learning
rate $1/n$, clamped to $[0.05,0.95]$ and then $[0.1,0.9]$. The update moves the biases towards the
newly drawn query rather than the evaluated one, so the fitness only decides whether they move.
Stage~4 is again true, and all $24\times256$ samples for it failed: 53\% left unsolved goals, 17\%
used unknown identifiers, 11\% did not parse and 8\% contained no extractable code. The run
generated 10{,}776 samples and 73.1M output tokens, and lasted 23.6 hours.

\emph{Row 21}. The algorithm is uniform random search with a
stored current search point. Stage~4 passed (2 of 256 samples), because in this pipeline version the
\texttt{NotRandomSearch} obligation did not require the witness state to be reachable: the
accepted proof uses a state with all biases equal to $1$, which the algorithm never reaches. The
final version of the pipeline requires the state witness to be reachable. The run went on to define
the potential $\Phi(x)=1-\onemax(x\oplus z)/n$ and passed Stage~9c three times, but all $15\times256$ samples for Stage~9d failed, although the
statement is true for this potential: 31\% omitted the required lemma, 25\% left unsolved goals,
11\% did not parse, 10\% used unknown identifiers, and 16\% recorded no Lean error. The run generated 22{,}584
samples and 124.3M output tokens, and lasted 36.8 hours.

\subsection{\onemax, Claude Opus 5 (staged; complete)}\label{app:run-onemax}

\begin{algorithm}[H]
\caption{$(1{+}1)$-style hill-climber with rate $1/(n+1)$ (\onemax, Claude Opus 5).}
\label{alg:app-onemax}
\begin{algorithmic}[1]
\Procedure{\init}{$u$}
  \State \Return $(\mathrm{bits}_{1/2}(u),\ \textsc{false},\ 0^n,\ 0)$
\EndProcedure
\Procedure{\step}{$u,\ (q,b,x^\ast,f^\ast),\ f$}
  \If{$b$ \textbf{and} $f<f^\ast$} \State $(c,c_f)\gets(x^\ast,f^\ast)$
  \Else \State $(c,c_f)\gets(q,f)$ \Comment{accept (always at the first step)}
  \EndIf
  \State $p_i\gets n/(n+1)$ if $c_i=1$, else $1/(n+1)$, for all $i$ \Comment{flip each bit w.p.\ $\tfrac1{n+1}$}
  \State \Return $(\mathrm{bits}_p(u),\ \textsc{true},\ c,\ c_f)$
\EndProcedure
\end{algorithmic}
\end{algorithm}

\noindent\emph{Potential.} With $M(s)=\max\{F_z(q),F_z(x^\ast)\}$ if $b$, else $M(s)=F_z(q)$,
\[
  \Phi_z(s)=[q\ne\bar z]\cdot\Bigl(\tfrac{1}{n+1}+\tfrac{n}{n+1}\bigl(1-\tfrac{M(s)}{n}\bigr)\Bigr).
\]
\noindent\emph{Runtime.} $\varepsilon_1=\frac1{n+1}$, $p_1=\frac{2}{(n+1)^2}$: $\expect{T_A}\le 2(n+1)^3=O(n^3)$.

\subsection{\leadingones, Claude Opus 5.5 (staged; complete)}\label{app:run-lo}

\begin{algorithm}[H]
\caption{Fitness-dependent mutation around an elitist reference point (\leadingones, Claude Opus 5.5).}
\label{alg:app-lo}
\begin{algorithmic}[1]
\Procedure{\init}{$u$}
  \State \Return $(\mathrm{bits}_{1/2}(u),\ \textsc{false},\ 0^n,\ 0)$
\EndProcedure
\Procedure{\step}{$u,\ (q,b,x^\ast,f^\ast),\ f$}
  \If{$b$ \textbf{and} $f<f^\ast$} \State $(c,c_f)\gets(x^\ast,f^\ast)$
  \Else \State $(c,c_f)\gets(q,f)$ \Comment{accept (always at the first step)}
  \EndIf
  \State $r\gets 1/(\max\{c_f,0\}+2)$ \Comment{rate read from the learned fitness}
  \State \Return $(c\oplus\mathrm{bits}_r(u),\ \textsc{true},\ c,\ c_f)$
\EndProcedure
\end{algorithmic}
\end{algorithm}

\noindent\emph{Potential.} With $M(s)=\max\{F(q),F(x^\ast)\}$ if $b$, else $M(s)=F(q)$,
\[
  \Phi(s)=[q\ne\bar z]\cdot\frac{n+1-M(s)}{n+1}.
\]
\noindent\emph{Runtime.} $\varepsilon_1=\frac1{n+1}$, $p_1=\frac1{(n+2)^2}$: $\expect{T_A}\le 4(n+1)(n+2)^2=O(n^3)$.

\subsection{\leadingones with the $e^{-1}$ lemma, Claude Opus 5.5 (staged; complete)}\label{app:run-lo2}

\init, \step\ and the potential are identical to \Cref{alg:app-lo}. The mask is
drawn with the per-bit sampler, using the same fitness-dependent rate for every bit.

\noindent\emph{Runtime.} $\varepsilon_1=\frac1{n+1}$, $p_1=\frac{e^{-1}}{n+2}$: $\expect{T_A}\le 4e(n+1)(n+2)=O(n^2)$.

\subsection{\jump$_k$, Claude Opus 5.5 (staged; complete)}\label{app:run-jump}

\begin{algorithm}[H]
\caption{$(1{+}1)$-style hill-climber with rate $1/(n+1)$ (\jump$_k$, Claude Opus 5.5).}
\label{alg:app-jump}
\begin{algorithmic}[1]
\Procedure{\init}{$u$}
  \State \Return $(\mathrm{bits}_{1/2}(u),\ \textsc{false},\ 0^n,\ 0)$
\EndProcedure
\Procedure{\step}{$u,\ (q,b,x^\ast,f^\ast),\ f$}
  \If{$b$ \textbf{and} $f<f^\ast$} \State $(c,c_f)\gets(x^\ast,f^\ast)$
  \Else \State $(c,c_f)\gets(q,f)$ \Comment{accept (always at the first step)}
  \EndIf
  \State \Return $(c\oplus\mathrm{bits}_{1/(n+1)}(u),\ \textsc{true},\ c,\ c_f)$
\EndProcedure
\end{algorithmic}
\end{algorithm}

\noindent\emph{Potential.} With $N=n+k$ and $M(s)=\max\{F_z(q),f^\ast\}$ if $b$, else $M(s)=F_z(q)$,
\[
  \Phi_z(s)=\begin{cases}0&\text{if } F_z(q)=F_z(\bar z),\\[2pt]
  \dfrac{N+1-\min\{\max\{M(s),0\},N\}}{N+1}&\text{otherwise.}\end{cases}
\]
\noindent\emph{Runtime.} $\varepsilon_1=\frac1{n+k+1}$, $p_1=(n+1)^{-(k+1)}$: $\expect{T_A}\le 4(n+k+1)(n+1)^{k+1}=O(n^{k+2})$.

\subsection{\jump$_k$ with the $e^{-1}$ lemma, Claude Opus 5.5 (staged, 3 generations; complete)}\label{app:run-jump-exp}

We supplied the $e^{-1}$ lemma and ran the pipeline for three
generations. The runtimes are $O(n^{k+1)}$ in all generations, but
they differ in the leading constant.  The agent used the same
potential as in \Cref{app:run-jump}. All three generations of the
algorithm keep a strictly elitist search point $x^\ast$ (a new query
is accepted only if $f>f^\ast$, or at the first step) and differ only
in the mask rate.

\emph{Generation 1.} The mask rate is reset to $1/(n+2)$ on a strict
improvement, multiplied by $1+1/n$ otherwise, and clamped to the
interval $[1/(n+2),2/(n+4)]$.

\begin{algorithm}[H]
\caption{Self-adjusting mask rate (\jump$_k$ with the $e^{-1}$ lemma, generation 1).}
\label{alg:app-jump-exp-1}
\begin{algorithmic}[1]
\Procedure{\init}{$u$}
  \State \Return $(\mathrm{bits}_{1/2}(u),\ \textsc{false},\ 0^n,\ 0,\ \frac1{n+2})$
\EndProcedure
\Procedure{\step}{$u,\ (q,b,x^\ast,f^\ast,r),\ f$}
  \If{$b$ \textbf{and} $f\le f^\ast$}
    \State $(c,c_f)\gets(x^\ast,f^\ast)$; \ $r'\gets\min\{\frac2{n+4},\ r(1+\frac1n)\}$
  \Else
    \State $(c,c_f)\gets(q,f)$; \ $r'\gets\frac1{n+2}$ \Comment{strict improvement or first step}
  \EndIf
  \State $\rho\gets\mathrm{clamp}_{[\frac1{n+2},\frac2{n+4}]}(r')$
  \State \Return $(c\oplus\mathrm{bits}_\rho(u),\ \textsc{true},\ c,\ c_f,\ r')$
\EndProcedure
\end{algorithmic}
\end{algorithm}

\noindent\emph{Runtime.} $\varepsilon_1=\frac1{n+k+1}$, $p_1=(n+2)^{-k}e^{-2}$:
$\expect{T_A}\le4e^2(n+k+1)(n+2)^k=O(n^{k+1})$.

\emph{Generation 2.} The mask rate is $1/(n+2)$ after a strict improvement and $2/(n+2)$ after a
failure.

\begin{algorithm}[H]
\caption{Failure-doubled mask rate (\jump$_k$ with the $e^{-1}$ lemma, generation 2).}
\label{alg:app-jump-exp-2}
\begin{algorithmic}[1]
\Procedure{\init}{$u$}
  \State \Return $(\mathrm{bits}_{1/2}(u),\ \textsc{false},\ 0^n,\ 0)$
\EndProcedure
\Procedure{\step}{$u,\ (q,b,x^\ast,f^\ast),\ f$}
  \If{$b$ \textbf{and} $f\le f^\ast$}
    \State $(c,c_f)\gets(x^\ast,f^\ast)$; \ $\rho\gets\frac2{n+2}$
  \Else
    \State $(c,c_f)\gets(q,f)$; \ $\rho\gets\frac1{n+2}$ \Comment{strict improvement or first step}
  \EndIf
  \State \Return $(c\oplus\mathrm{bits}_\rho(u),\ \textsc{true},\ c,\ c_f)$
\EndProcedure
\end{algorithmic}
\end{algorithm}

\noindent\emph{Runtime.} $\varepsilon_2=\frac1{2(n+k+1)}$, $p_2=\bigl(\frac2{n+2}\bigr)^k e^{-2}$:
$\expect{T_A}\le 2^{3-k}e^2(n+k+1)(n+2)^k$.

\emph{Generation 3.} A phase bit is reset on a strict improvement and toggled on every failure; the
mask rate is $3/(n+2)$ if the updated phase is set and $2/(n+2)$ otherwise.

\begin{algorithm}[H]
\caption{Rate cycling with a learned phase bit (\jump$_k$ with the $e^{-1}$ lemma, generation 3).}
\label{alg:app-jump-exp-3}
\begin{algorithmic}[1]
\Procedure{\init}{$u$}
  \State \Return $(\mathrm{bits}_{1/2}(u),\ \textsc{false},\ 0^n,\ 0,\ \textsc{false})$
\EndProcedure
\Procedure{\step}{$u,\ (q,b,x^\ast,f^\ast,\phi),\ f$}
  \If{$b$ \textbf{and} $f\le f^\ast$}
    \State $(c,c_f)\gets(x^\ast,f^\ast)$; \ $\phi'\gets\neg\phi$
  \Else
    \State $(c,c_f)\gets(q,f)$; \ $\phi'\gets\textsc{false}$ \Comment{strict improvement or first step}
  \EndIf
  \State $\rho\gets\frac3{n+2}$ if $\phi'$, else $\rho\gets\frac2{n+2}$
  \State \Return $(c\oplus\mathrm{bits}_\rho(u),\ \textsc{true},\ c,\ c_f,\ \phi')$
\EndProcedure
\end{algorithmic}
\end{algorithm}

\noindent\emph{Runtime.} $\varepsilon_3=\frac1{n+k+1}$,
$p_3=\min\bigl\{(\frac2{n+2})^k e^{-2},\ (\frac3{n+2})^k e^{-n/((n+2)/3-1)}\bigr\}$:
$\expect{T_A}\le4(n+k+1)/p_3=O(n^{k+1})$.

\subsection{\cycleuncover, Claude Fable 5.1 (staged; complete)}\label{app:run-laminate}

Write $H=2\lfloor n/2\rfloor$ and $\rho=1/(n+1)$. The state $(q,A,B,r,a,k)$ holds the query, the
indicator strings $A,B$ of the vertices labelled with the two colours so far, the frontier (last
labelled vertex) $r$, its colour $a$, and the number $k$ of labelled vertices.

\begin{algorithm}[H]
\caption{Cycle walk (\cycleuncover, Claude Fable 5.1).}\label{alg:app-laminate}
\begin{algorithmic}[1]
\Procedure{\init}{$u$}
  \State \Return $(\mathrm{bits}_\rho(u),\ 0^n,\ 0^n,\ 0^n,\ \textsc{false},\ 0)$
\EndProcedure
\Procedure{\step}{$u,\ (q,A,B,r,a,k),\ f$}
  \State $\pi\gets A$ if $k=H$, else $r$; \quad $T\gets\pi\oplus q$ \Comment{the mask just used}
  \If{$f=2$ \textbf{and} $0<k<H$ \textbf{and} $|T|_1=1$ \textbf{and} $T\wedge(A\oplus B)=0^n$}
    \If{$a$} \State $B\gets B\oplus T$ \Else \State $A\gets A\oplus T$ \EndIf
    \State $r\gets T$; \ $a\gets\neg a$; \ $k\gets k+1$ \Comment{$T$ is adjacent to the frontier}
  \ElsIf{$f=H+1$ \textbf{and} $k=0$ \textbf{and} $|T|_1=1$}
    \State $A\gets T$; \ $r\gets T$; \ $a\gets\textsc{true}$; \ $k\gets1$ \Comment{$T$ is a read coordinate}
  \EndIf
  \State $\pi'\gets A$ if $k=H$, else $r$
  \State \Return $(\pi'\oplus\mathrm{bits}_\rho(u),\ A,\ B,\ r,\ a,\ k)$
\EndProcedure
\end{algorithmic}
\end{algorithm}

\noindent\emph{Potential.} With $k^+(s)$ the label count after one \step\ from $s$ with the true fitness and truncated
subtraction $\dot-$: $\Phi_z(s)=0$ if $F_\sigma(q)=3h$, else $\Phi_z(s)=\bigl(1+(H\mathbin{\dot-}k^+(s))\bigr)/(H+2)$.
\noindent\emph{Runtime.} $\varepsilon_1=\frac1{H+2}$, $p_1=\frac1{3(n+1)}$: $\expect{T_A}\le12(n+1)(2\lfloor n/2\rfloor+2)=O(n^2)$.

\subsection{\cycleuncover, Claude Fable 5.1 (zero-shot ablation; complete)}\label{app:run-laminate-zs}

This algorithm was obtained in an ablation run which did not include
the staged pipeline. The algorithm samples a bitstring uniformly at
random in each iteration, i.e., it is essentially identical to the
baseline algorithm provided.

\begin{algorithm}[H]
\caption{Random search with a recorded incumbent (\cycleuncover, zero-shot, Claude Fable 5.1).}
\label{alg:app-laminate-zs}
\begin{algorithmic}[1]
\Procedure{\init}{$z$}
  \State \Return $(\mathrm{bits}_{1/2}(z),\ \textsc{false},\ 0^n,\ 0)$
\EndProcedure
\Procedure{\step}{$z,\ (x',b,x,f_x),\ f$}
  \If{\textbf{not} $b$ \textbf{or} $f\ge f_x$} \State $(x,f_x)\gets(x',f)$ \EndIf
  \State \Return $(\mathrm{bits}_{1/2}(z),\ \textsc{true},\ x,\ f_x)$ \Comment{the query ignores the state}
\EndProcedure
\end{algorithmic}
\end{algorithm}

\noindent\emph{Potential.} No potential is constructed because the agent does not use the staged pipeline.
\emph{Runtime.} $\expect{T_A}\le2^n/2$.

\subsection{\cycleuncover, Claude Sonnet (staged; incomplete)}\label{app:run-laminate-sonnet}
Claude Sonnet did not complete the pipeline with problem card \cycleuncover. The following is a description of the algorithm
constructed in generation 4 of 5. The algorithm is (1+1) EA, except that the current search point
is accepted only if the mutant is strictly better than the parent.

\begin{algorithm}[H]
\caption{Design 4 (\cycleuncover, Claude Sonnet).}\label{alg:app-laminate-sonnet}
\begin{algorithmic}[1]
\Procedure{\init}{$z$}
  \State \Return $(\mathrm{bits}_{1/2}(z),\ 0^n,\ 2h)$ \Comment{$0^n$ has fitness $2h$}
\EndProcedure
\Procedure{\step}{$z,\ (x',x,f_x),\ f$}
  \If{$f>f_x$} \State $(x,f_x)\gets(x',f)$ \EndIf \Comment{Design 5: also accept ties if $f_x>2h$}
  \State $p_i\gets(n-1)/n$ if $x_i=1$, else $1/n$, for all $i$
  \State \Return $(\mathrm{bits}_p(z),\ x,\ f_x)$
\EndProcedure
\end{algorithmic}
\end{algorithm}

\noindent\emph{Potential.} $\Phi_z(s)=\bigl(3h-F_\sigma(q)\bigr)/(3h)$.
\noindent\emph{Runtime.} None proved (stopped at Stage 10h); $\varepsilon_1=1/(3h)$, and $p_1=1/(en)$.

\subsection{\vertexcover, Claude Fable 5.1 (staged; complete)}\label{app:run-vc}

Let $N=n^2$, $\rho=1/(n+1)$, $u_j=\lfloor j/n\rfloor\bmod n$, and $v_j=j\bmod n$. Let $\mathrm{probe}(j)$ be
$1^n$ with bits $u_j,v_j$ cleared. Let $\mathrm{opt}(A)$ be a minimum-size string covering every edge
recorded in the table $A$, and $\mathrm{target}(k,A)=\mathrm{probe}(k)$ if $k<N$, else $\mathrm{opt}(A)$.

\begin{algorithm}[H]
\caption{Edge probing, then solve (\vertexcover, Claude Fable 5.1).}\label{alg:app-vc}
\begin{algorithmic}[1]
\Procedure{\init}{$u$}
  \State \Return $(\mathrm{probe}(0)\oplus\mathrm{bits}_\rho(u),\ 0,\ A_\emptyset)$ \Comment{empty edge table}
\EndProcedure
\Procedure{\step}{$u,\ (q,k,A),\ f$}
  \If{$q=\mathrm{target}(k,A)$} \Comment{hit: the probe was queried exactly}
    \State $A[u_k,v_k]\gets[f<-n]$ \Comment{probe is a non-cover iff $u_kv_k$ is an edge}
    \If{$k<N$} \State $k\gets k+1$ \EndIf
  \EndIf
  \State \Return $(\mathrm{target}(k,A)\oplus\mathrm{bits}_\rho(u),\ k,\ A)$
\EndProcedure
\end{algorithmic}
\end{algorithm}

\noindent\emph{Potential.} For non-optimal states, the potential is $\Phi_z(s)=\frac12+U(s)/\bigl(2(2N+1)\bigr)$ where
$U(s)=2\bigl(N-\min(k,N)\bigr)+[q\ne\mathrm{target}(k,A)]$.
\noindent\emph{Runtime.} $\varepsilon_1=\frac1{2(2n^2+1)}$, $p_1=\frac1{n+1}$: $\expect{T_A}\le8(2n^2+1)(n+1)=O(n^3)$.
The bound counts fitness evaluations only; $\mathrm{opt}(A)$ is defined non-constructively (as some
minimum cover of the learned graph), so the computation between queries is not bounded.

\subsection{\bandwidthvc, Claude Fable 5.1 (staged, flat fitness; complete)}\label{app:run-bvc}

Let $K=k+1$, $L=2nK$, $a_p=\lfloor p/K\rfloor$, $b_p=a_p+1+(p\bmod K)$, let $\mathrm{probe}(p)$ be
$1^n$ with bits $a_p,b_p$ set to 0. Note that $\mathrm{probe}(p)$ is a vertex cover if and only if $a_pb_p$ is not an edge in the graph. Let $\mathrm{opt}(A)$ be as in \Cref{app:run-vc}, $\mathrm{plan}(t,A)=\mathrm{probe}(t/2)$
if $t$ is even and $t<L$, else $\mathrm{opt}(A)$, and $\beta(y)_i=n/(n+1)$ if $y_i=1$, else $1/(n+1)$.

\begin{algorithm}[H]
\caption{Interleaved window probing and solving (\bandwidthvc, Claude Fable 5.1).}\label{alg:app-bvc}
\begin{algorithmic}[1]
\Procedure{\init}{$u$}
  \State \Return $(\mathrm{bits}_{\beta(\mathrm{plan}(0,A_\emptyset))}(u),\ 0,\ A_\emptyset)$
\EndProcedure
\Procedure{\step}{$u,\ (q,t,A),\ f$}
  \If{$q=\mathrm{plan}(t,A)$ \textbf{and} $t<L$}
    \If{$t$ even} \State $A[a_{t/2},b_{t/2}]\gets[f<-n]$ \EndIf
    \State $t\gets t+1$
  \EndIf
  \State \Return $(\mathrm{bits}_{\beta(\mathrm{plan}(t,A))}(u),\ t,\ A)$
\EndProcedure
\end{algorithmic}
\end{algorithm}

\noindent\emph{Potential.} Off the optimum, $\Phi_z(s)=\bigl(2(L\mathbin{\dot-}t)+[q\ne\mathrm{plan}(t,A)]+1\bigr)/(2L+2)$.
\noindent\emph{Runtime.} $\varepsilon_1=\frac1{2L+2}$, $p_1=\frac14$: $\expect{T_A}\le16(2L+2)=64n(k+1)+32=O(kn)$.
The bound counts fitness evaluations only; $\mathrm{opt}(A)$ is defined non-constructively (as some
minimum cover of the learned graph), so the computation between queries is not bounded.

\subsection{\bandwidthvc, Claude Fable 5.1 (staged, graded fitness; incomplete)}\label{app:run-bvc2}

Let $\kappa=\mathrm{clamp}_{[\frac1{2(n+4)},\frac2{n+4}]}$.

\begin{algorithm}[H]
\caption{Direction-adaptive asymmetric mutation (\bandwidthvc, Claude Fable 5.1).}\label{alg:app-bvc2}
\begin{algorithmic}[1]
\Procedure{\init}{$u$}
  \State \Return $(\mathrm{bits}_{1/2}(u),\ \textsc{false},\ 0^n,\ 0,\ \frac1{n+4},\ \frac1{n+4})$
\EndProcedure
\Procedure{\step}{$u,\ (q,b,a,f_a,\rho_{\text{down}},\rho_{\text{up}}),\ f$}
  \If{\textbf{not} $b$} \State $(a,f_a)\gets(q,f)$
  \ElsIf{$f\ge f_a$}
    \If{$|q|_1<|a|_1$} \State $\rho_{\text{down}}\gets\kappa(2\rho_{\text{down}})$; \ $\rho_{\text{up}}\gets\kappa(\rho_{\text{up}}/2)$
    \Else \State $\rho_{\text{down}}\gets\kappa(\rho_{\text{down}}/2)$; \ $\rho_{\text{up}}\gets\kappa(2\rho_{\text{up}})$
    \EndIf
    \State $(a,f_a)\gets(q,f)$
  \EndIf
  \State $p_i\gets\kappa(\rho_{\text{down}})$ if $a_i=1$, else $\kappa(\rho_{\text{up}})$, for all $i$
  \State \Return $(a\oplus\mathrm{bits}_p(u),\ \textsc{true},\ a,\ f_a,\ \rho_{\text{down}},\ \rho_{\text{up}})$
\EndProcedure
\end{algorithmic}
\end{algorithm}

\noindent\emph{Potential.} With $e(s)=q$ if $\neg b$ or $F_z(q)\ge F_z(a)$, else $a$: off the optimum
$\Phi_z(s)=\frac12+\frac{-F_z(e(s))}{4n}$ if $e(s)$ is a cover and $\Phi_z(s)=\frac34+\frac{\mathrm{unc}_z(e(s))}{4(kn+1)}$ otherwise.
\noindent\emph{Runtime.} None proved (stopped at Stage 10h); $\varepsilon_1=\frac1{4((k+1)n+1)}$, $p_1=\frac1{16(n+4)}$.

\subsection{\bandwidthvc, Claude Fable 5.1 (zero-shot ablation; incomplete)}\label{app:run-bvc-zs}

\begin{algorithm}[H]
\caption{$(1{+}1)$ EA with rate $1/n$ (\bandwidthvc, zero-shot, Claude Fable 5.1).}\label{alg:app-bvc-zs}
\begin{algorithmic}[1]
\Procedure{\init}{$u$}
  \State \Return $(\mathrm{bits}_{1/2}(u),\ \textsc{false},\ 0^n,\ 0)$
\EndProcedure
\Procedure{\step}{$u,\ (q,b,x^\ast,f^\ast),\ f$}
  \If{$b$ \textbf{and} $f<f^\ast$} \State $(c,c_f)\gets(x^\ast,f^\ast)$ \Else \State $(c,c_f)\gets(q,f)$ \EndIf
  \State \Return $(c\oplus\mathrm{bits}_{1/n}(u),\ \textsc{true},\ c,\ c_f)$
\EndProcedure
\end{algorithmic}
\end{algorithm}

\noindent\emph{Potential.} None. \emph{Runtime.} The bound $\expect{T_A}\le2^n/n$ was left as \texttt{sorry}.

\subsection{\bandwidthvc, Claude Fable 5.1, second session (staged; incomplete)}\label{app:run-bvc-2nd}

\begin{algorithm}[H]
\caption{Fill-then-prune mutation (\bandwidthvc, Claude Fable 5.1).}\label{alg:app-bvc-2nd}
\begin{algorithmic}[1]
\Procedure{\init}{$u$}
  \State \Return $(\mathrm{bits}_{1/2}(u),\ \textsc{false},\ 0^n,\ 0)$
\EndProcedure
\Procedure{\step}{$u,\ (q,b,x^\ast,f^\ast),\ f$}
  \If{$b$ \textbf{and} $f<f^\ast$} \State $(c,c_f)\gets(x^\ast,f^\ast)$ \Else \State $(c,c_f)\gets(q,f)$ \EndIf
  \If{$c_f<-n$} \Comment{incumbent is not a cover: fill}
    \State $p_i\gets1/(n+1)^2$ if $c_i=1$, else $1/2$, for all $i$
  \Else \Comment{incumbent is a cover: prune}
    \State $p_i\gets1/(n+1)$ for all $i$
  \EndIf
  \State \Return $(c\oplus\mathrm{bits}_p(u),\ \textsc{true},\ c,\ c_f)$
\EndProcedure
\end{algorithmic}
\end{algorithm}

\noindent\emph{Potential.} Off the optimum, $\Phi_z(s)=\bigl(1+[\neg b]+(-F_z(x^\ast))/((n+1)(kn+1))\bigr)/3$.
\noindent\emph{Runtime.} None proved (stopped at Stage 10c); $\varepsilon_1=\frac1{3(n+1)(kn+1)}$, $p_1=(n+1)^{-n}$.

\subsection{\bandwidthvc, Claude Opus 5.5 (staged; complete)}\label{app:run-bvc-opus}

Let $L=2k+1$ and $\ell=\min\{L,n\}$.

\begin{algorithm}[H]
\caption{Window-repair random search (\bandwidthvc, Claude Opus 5.5, attempt 2).}\label{alg:app-bvc-opus}
\begin{algorithmic}[1]
\Procedure{\init}{$u$}
  \State \Return $(\mathrm{bits}_{1/2}(u),\ \textsc{false},\ 0^n,\ 0)$
\EndProcedure
\Procedure{\step}{$u,\ (q,b,x^\ast,f^\ast),\ f$}
  \If{$b$ \textbf{and} $f<f^\ast$} \State $(c,c_f)\gets(x^\ast,f^\ast)$ \Else \State $(c,c_f)\gets(q,f)$ \EndIf
  \State $y\gets\mathrm{bits}_{1/2}(u)$; \ $m\gets\mathrm{bits}_{1/3}(u^{+n})$
  \If{$y_i=1$ for all $i<\ell$} \Comment{an optimum never chooses all of $0,\dots,2k$}
    \State $y_i\gets c_i\oplus m_i$ for all $i<\ell$ \Comment{redraw the window near the incumbent}
  \EndIf
  \State \Return $(y,\ \textsc{true},\ c,\ c_f)$
\EndProcedure
\end{algorithmic}
\end{algorithm}

\noindent\emph{Potential.} With $F^{\mathrm{opt}}_z=\max_yF_z(y)$, $e(s)=q$ if $\neg b$ or $F_z(q)\ge F_z(x^\ast)$, else $x^\ast$,
$\eta=1/(4\cdot3^{L})$ and $R=(n+1)kn+n$: off the optimum
$\Phi_z(s)=(1-\eta)+\eta\min\bigl\{1,(F^{\mathrm{opt}}_z-F_z(e(s)))/(R+1)\bigr\}$.
\noindent\emph{Runtime.} $\varepsilon_1=1-\eta$, $p_1=2^{-n}(1+3^{-L})$: $\expect{T_A}\le 4\cdot2^n/\bigl((1-a/4)(1+a)\bigr)=\Theta(2^n)$
with $a=3^{-(2k+1)}$.

The algorithm is close to random search with expected runtime bound
$O(2^n)$.  The algorithm maintains a currently best search point $x^*$
with corresponding fitness $f^*$. In each iteration, the algorithm
samples uniformly at random a new search point $y$. If all the first
$\ell$ bits of $y$ are 1-bits, then each of these bits is copied from
the current search point $x^*$, and then mutated with probability
$1/3$.

\subsection{\submodmin, integer-valued, Claude Opus 5.5 (staged; complete)}\label{app:run-submod}

The problem card contains the following advice (verbatim apart from a corrected typo): ``Note that
this is a challenging problem domain where naive algorithmic designs, such as elitist hill-climbing
with a single individual, are unlikely to succeed. Current efforts in designing efficient evolutionary
algorithms for this problem are ongoing, so your design must be innovative and well-justified. You
need to carefully consider the properties of submodular functions and how to exploit them in your
algorithm design. Note that the minimisation [of] sub-modular functions is a well-studied problem in
combinatorial optimisation which allows for polynomial time algorithms.''

The state $(q,b,\lambda,k,g,f_{\mathrm{prev}},t,\phi)$ holds the query $q$, the initialisation flag
$b$, a point $\lambda\in[0,1]^n$, the position $k$ in the current chain, the accumulated subgradient
$g\in\R^n$, the previous chain value $f_{\mathrm{prev}}$, the intended query $t$ and a flag $\phi$
recording that an optimum has been seen. $S_k(\lambda)$ is the set of the $k$ coordinates with the
largest values of $\lambda$ (ties broken by index), $\eta=1/(nB^2)$ and $m=\mathrm{bits}_{1/(n+1)}(u)$.

\begin{algorithm}[H]
\caption{Projected subgradient descent on the Lov\'asz extension (\submodmin, Claude Opus 5.5).}
\label{alg:app-submod}
\begin{algorithmic}[1]
\Procedure{\init}{$u$}
  \State \Return $(S_0\oplus m,\ \textsc{false},\ (\tfrac12)^n,\ 0,\ 0^n,\ 0,\ S_0,\ \textsc{false})$
\EndProcedure
\Procedure{\step}{$u,\ (q,b,\lambda,k,g,f_{\mathrm{prev}},t,\phi),\ f$}
  \If{\textbf{not} $b$} \State $t\gets S_0(\lambda)$; $k\gets0$ \Comment{start the first chain}
  \ElsIf{$\phi$} \State no update \Comment{$t$ is the recorded optimum; the next query is again $t\oplus m$}
  \ElsIf{$f\ge0$} \State $t\gets q$; $\phi\gets\textsc{true}$ \Comment{an optimum was queried}
  \ElsIf{$q=t$} \Comment{the intended level set was queried exactly}
    \State $g_i\gets f_{\mathrm{prev}}-f$ for the coordinate $i$ added at position $k$ \Comment{marginal value}
    \If{$k<n$} \State $k\gets k+1$; $f_{\mathrm{prev}}\gets f$; $t\gets S_k(\lambda)$
    \Else \State $\lambda\gets\mathrm{clamp}_{[0,1]}(\lambda-\eta g)$; $k\gets0$; $g\gets0^n$; $t\gets S_0(\lambda)$ \Comment{subgradient step}
    \EndIf
  \EndIf \Comment{otherwise the query was perturbed: repeat $t$}
  \State \Return $(t\oplus m,\ \textsc{true},\ \lambda,\ k,\ g,\ f_{\mathrm{prev}},\ t,\ \phi)$
\EndProcedure
\end{algorithmic}
\end{algorithm}

\noindent\emph{Potential.} With $y^*$ a fixed minimiser of $f$, $D(s)=\|\lambda-\mathbf 1_{y^*}\|^2$ and $c=\eta/(2n+4)$:
$\Phi(s)=0$ if $F_f(q)\ge0$; otherwise $\Phi(s)=c/(n+1)$ if $\phi$, and
$\Phi(s)=\bigl(\min\{D(s),n\}+c\,(2(n+1-k)-[q=t])\bigr)/(n+1)$ if not ($c(2n+3)$ replaces the
second term before initialisation).
\noindent\emph{Runtime.} Generation 1: $\varepsilon_1=1/\bigl(nB^2(2n+4)(n+1)\bigr)$, $p_1=\frac13$:
$\expect{T_A}\le12\,nB^2(2n+4)(n+1)=O(n^{3+2c})$. Generations 2--5 give $16\,n^2B^2(n+2)$,
$32\,nB^2(n+1)$, $\frac{128}{15}B^2(n+1)^2$ and $\frac{512}{63}B^2(n+1)^2=O(n^{2+2c})$.

\subsection{\wellcoveredvc, anonymised card, Claude Fable 5.1 (staged; incomplete)}\label{app:run-wcvc-fable}

Let $\kappa(x)=\mathrm{clamp}_{[\frac1{4n},\frac1{2n}]}(x)$ (applied only when sampling).

\begin{algorithm}[H]
\caption{Polarity-adaptive asymmetric mutation (\texttt{wcvc}, Claude Fable 5.1).}\label{alg:app-wcvc-fable}
\begin{algorithmic}[1]
\Procedure{\init}{$u$}
  \State \Return $(\mathrm{bits}_{1/2}(u),\ \textsc{false},\ 0^n,\ 0,\ \frac1{2n},\ \frac1{2n})$
\EndProcedure
\Procedure{\step}{$u,\ (q,b,x^\ast,f^\ast,r_{\text{up}},r_{\text{down}}),\ f$}
  \If{\textbf{not} $b$} \State $(c,c_f)\gets(q,f)$
  \ElsIf{$f\ge f^\ast$}
    \If{$|q|_1>|x^\ast|_1$} \State $(r_{\text{up}},r_{\text{down}})\gets(2r_{\text{up}},r_{\text{down}}/2)$
    \ElsIf{$|q|_1<|x^\ast|_1$} \State $(r_{\text{up}},r_{\text{down}})\gets(r_{\text{up}}/2,2r_{\text{down}})$
    \EndIf
    \State $(c,c_f)\gets(q,f)$
  \Else \State $(c,c_f)\gets(x^\ast,f^\ast)$
  \EndIf
  \State $p_i\gets\kappa(r_{\text{down}})$ if $c_i=1$, else $\kappa(r_{\text{up}})$, for all $i$
  \State \Return $(c\oplus\mathrm{bits}_p(u),\ \textsc{true},\ c,\ c_f,\ r_{\text{up}},\ r_{\text{down}})$
\EndProcedure
\end{algorithmic}
\end{algorithm}

\noindent\emph{Potential.} With $F^{\mathrm{eff}}_z(s)=\max\{F_z(x^\ast),F_z(q)\}$ if $b$, else $F_z(q)$, and
$R=(n+1)(n^2+1)+1$: off the optimum $\Phi_z(s)=\min\{1,(1-F^{\mathrm{eff}}_z(s))/R\}$.
\noindent\emph{Runtime.} None proved (stopped at Stage 10h); $\varepsilon_1=1/R$, $p_1=\frac1{8n}$.

\subsection{\wellcoveredvc, anonymised card, Claude Opus 5.5 (staged; complete)}\label{app:run-wcvc-opus}

\begin{algorithm}[H]
\caption{Regime-switching asymmetric mutation (\texttt{wcvc}, Claude Opus 5.5).}\label{alg:app-wcvc-opus}
\begin{algorithmic}[1]
\Procedure{\init}{$u$}
  \State \Return $(\mathrm{bits}_{1/2}(u),\ \textsc{false},\ 0^n,\ 0)$
\EndProcedure
\Procedure{\step}{$u,\ (q,b,x^\ast,f^\ast),\ f$}
  \If{$b$ \textbf{and} $f<f^\ast$} \State $(c,c_f)\gets(x^\ast,f^\ast)$ \Else \State $(c,c_f)\gets(q,f)$ \EndIf
  \If{$c_f\ge-n$} \Comment{incumbent is a cover: shrink}
    \State $p_i\gets\frac1{n+1}$ if $c_i=1$, else $\frac1{(n+1)^2}$, for all $i$
  \Else \Comment{uncovered edges remain: grow}
    \State $p_i\gets\frac1{(n+1)^2}$ if $c_i=1$, else $\frac1{n+1}$, for all $i$
  \EndIf
  \State \Return $(c\oplus\mathrm{bits}_p(u),\ \textsc{true},\ c,\ c_f)$
\EndProcedure
\end{algorithmic}
\end{algorithm}

\noindent\emph{Potential.} With $c(s)=q$ if $\neg b$ or $F_z(q)\ge f^\ast$, else $x^\ast$, and $N=(n+1)\binom{n+1}{2}+n+2$:
off the optimum $\Phi_z(s)=(1-F_z(c(s)))/N$.
\noindent\emph{Runtime.} $\varepsilon_1=\frac1N$, $p_1=\frac1{n+1}\bigl(\frac n{n+1}\bigr)^{n-1}$:
$\expect{T_A}\le4N(n+1)\bigl(\frac{n+1}n\bigr)^{n-1}=O(n^4)$.

\subsection{\wellcoveredvc, named card, Claude Fable 5.1 (staged; complete)}\label{app:run-wc}

\begin{algorithm}[H]
\caption{Fitness-directed asymmetric mutation (\wellcoveredvc, Claude Fable 5.1).}\label{alg:app-wc}
\begin{algorithmic}[1]
\Procedure{\init}{$u$}
  \State $q\gets\mathrm{bits}_{1/2}(u)$; \ \Return $(q,\ \textsc{false},\ q,\ 0,\ \textsc{true})$
\EndProcedure
\Procedure{\step}{$u,\ (q,b,x^\ast,f^\ast,\gamma),\ f$}
  \If{$b$ \textbf{and} $f<f^\ast$} \State $(c,c_f)\gets(x^\ast,f^\ast)$ \Else \State $(c,c_f)\gets(q,f)$ \EndIf
  \State $\gamma'\gets[c_f<-n]$ \Comment{grow while the incumbent is not a cover}
  \If{$\gamma'$} \State $p_i\gets\frac1{2(n+1)}$ if $c_i=1$, else $\frac1{n+1}$, for all $i$
  \Else \State $p_i\gets\frac1{n+1}$ if $c_i=1$, else $\frac1{2(n+1)}$, for all $i$
  \EndIf
  \State \Return $(c\oplus\mathrm{bits}_p(u),\ \textsc{true},\ c,\ c_f,\ \gamma')$
\EndProcedure
\end{algorithmic}
\end{algorithm}

\noindent\emph{Potential.} Off the optimum, $\Phi_z(s)=\bigl(1+\min\{\ell_z(q),\ell_z(x^\ast)\}\bigr)/(n^2+n+2)$, where
$\ell_z(x)=|x|_1-m^\ast_z$ for a cover $x$ and $\ell_z(x)=n+1+\mathrm{unc}_z(x)$ otherwise, with $m^\ast_z$
the minimum cover size.
\noindent\emph{Runtime.} $\varepsilon_1=\frac1{n^2+n+2}$, $p_1=\frac1{(n+1)^2}$: $\expect{T_A}\le4(n^2+n+2)(n+1)^2=O(n^4)$.

\section{Pipeline progress and per-stage statistics}\label{app:stage-stats}

\Cref{fig:completion} shows how each run of \Cref{tab:runs} progressed through the stages of the
pipeline, and \Cref{fig:stage-stats} where the runs spend their effort.

\begin{figure}[H]
  \centering
  \includegraphics[width=0.85\linewidth]{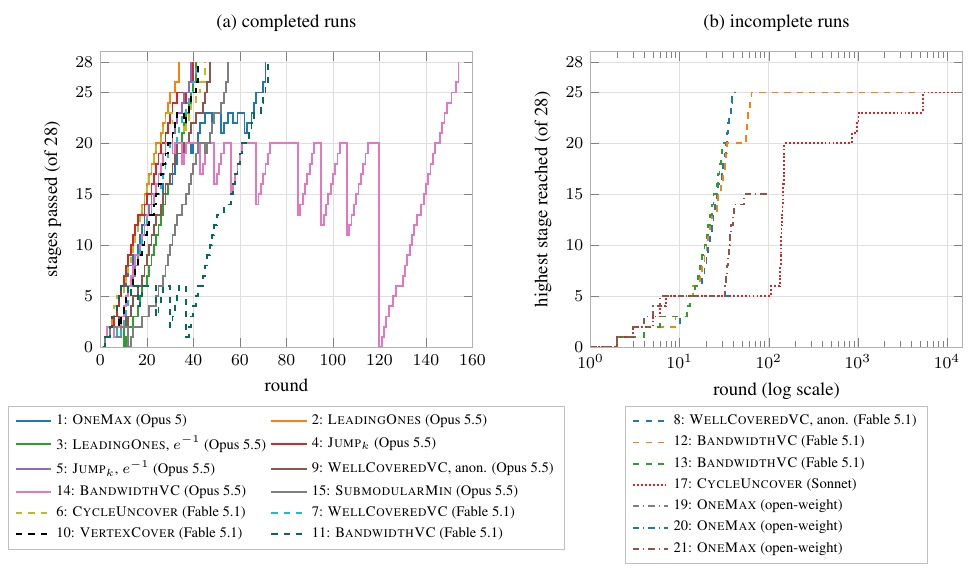}
  \caption{Progress through the 28 pipeline stages for the runs of \Cref{tab:runs}, numbered as in
    the table, one point per round. (a) Runs that completed the pipeline: the number of stages
    passed, including the drops caused by backtracking; the final point at 28 marks completion.
    Row 14 includes a first attempt that stalled at Stage~10c before the pipeline restarted.x1
    (b) Runs that did not complete: the highest stage reached so far, on a logarithmic round axis
    (rows 12 and 17 continue for thousands of automatic rounds after their best level). The
    zero-shot runs (rows 16 and 18) have a single stage and are omitted.}
  \label{fig:completion}
\end{figure}

\begin{figure}[H]
  \centering
  \includegraphics[width=0.85\linewidth]{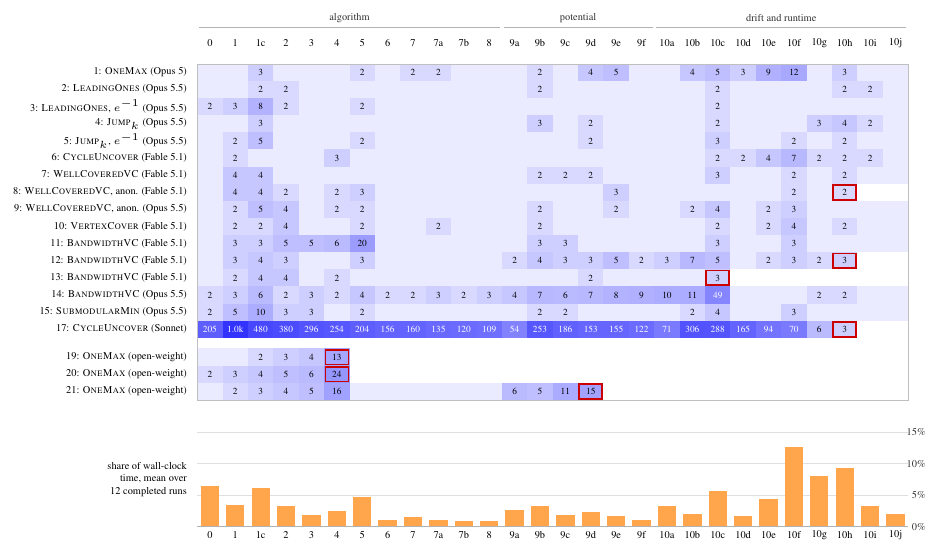}
  \caption{Per-stage statistics of the runs of \Cref{tab:runs}, numbered as in the table. Top: number
    of attempts per stage (colour on a logarithmic scale; the number is shown when a stage needed
    more than one attempt); a red frame marks the stage at which an incomplete run stopped. For
    rows 19--21 (vLLM backend) one attempt is a batch of 256 samples. Bottom: share of wall-clock
    time spent in each stage, averaged over the 12 completed frontier-model runs. Each run is counted
    up to its first completion or, if it did not complete, up to its last round at its highest
    stage. The zero-shot runs (rows 16 and 18) have a single stage and are omitted.}
  \label{fig:stage-stats}
\end{figure}

\Cref{fig:stage-stats} shows where the runs spend their effort. Over the 12 completed
frontier-model runs, the drift and runtime stages (Stages~10a--10j) account for 53\% of the wall-clock
time, algorithm construction (Stages 0--8) for 34\% and the potential (Stages 9a--9f) for 13\%.
The most expensive stages are the pointwise drift and mass proofs: \texttt{PositiveDrift} (10f,
13\%), \texttt{DriftSetMass} (10h, 9\%) and \texttt{NegativeDriftAndJump} (10g, 8\%). Two stages
fail at the first attempt in almost every run for technical rather than mathematical reasons:
\texttt{DriftConstantProperties} (10c) failed first in all 12 completed runs (81 attempts and 13
backtracks in total), mostly on the normal form of the baseline constants, and
\texttt{StepMeasurability} (1c) in 11 of them. The incomplete frontier-model runs stop at a drift
stage: rows 8, 12 and 17 at the mass bound (10h), which is false for their designs, and row 13 at
10c. The open-weight models stop much earlier, at \texttt{NotRandomSearch} (Stage~4, rows 19 and
20) and \texttt{PotentialFineGrained} (Stage~9d, row 21), after 13--24 batches of 256 samples at
that stage. The Sonnet run (row 17) needed between 54 and about 1{,}000 attempts at every stage it
passed before Stage 10g.

\end{document}